\documentclass[10pt, twocolumn]{article}

\usepackage[margin=1.5cm]{geometry}
\usepackage{lmodern}
\usepackage[T1]{fontenc}
\usepackage[utf8]{inputenc}
\usepackage{microtype}
\usepackage{titlesec}

\titleformat{\section}
  {\large\bfseries}
  {\thesection}{0.5em}{}

\titleformat{\subsection}
  {\normalsize\bfseries}
  {\thesubsection}{0.5em}{}

\usepackage{etoolbox}
\AtBeginEnvironment{thebibliography}{\footnotesize}

\renewenvironment{abstract}
{
  \par\vspace{0.5em}
  \noindent\textbf{Abstract ---}
  \bfseries
}
{\par\vspace{1em}}

\usepackage{booktabs}
\usepackage{amsmath,amsfonts,amssymb,amsthm,mathtools,bm}
\mathtoolsset{showonlyrefs}

\usepackage[linesnumbered,ruled,vlined]{algorithm2e}
\usepackage{graphicx}
\usepackage{array}
\usepackage{multirow}
\usepackage{textcomp}
\usepackage{xcolor}
\usepackage{url}
\usepackage{cite}
\usepackage{balance}
\usepackage{enumitem}
\usepackage{soul}
\usepackage[normalem]{ulem}
\usepackage{tcolorbox}
\usepackage{subcaption}
\usepackage{tikz}

\usetikzlibrary{decorations.pathreplacing}
\usetikzlibrary{calc}
\usetikzlibrary{spy}

\usepackage{csquotes}
\usepackage{hyperref}

\usepackage[
  skip=.3\baselineskip plus 2pt,
  parfill=20pt
]{parskip}

\makeatletter
\def\thm\@space\@setup{%
  \thm\@preskip=\parskip
  \thm\@postskip=0pt
}
\makeatother

\hypersetup{
    bookmarksopen=true,
    pdftitle={HOT-POT: Optimal Transport for Sparse Stereo Matching},
    pdfauthor={Antonin Clerc, Michael Quellmalz, Moritz Piening, Philipp Flotho, Gregor Kornhardt, Gabriele Steidl},
    pdfsubject={Accepted manuscript},
    pdfkeywords={stereo matching, optimal transport, epipolar geometry, sparse matching},
    colorlinks=true,
    linkcolor=blue,
    citecolor=blue,
    plainpages=false
}

\newtheorem{theorem}{Theorem}[section]

\newtheorem{proposition}[theorem]{Proposition}

\newtheorem{remark}[theorem]{Remark}

\theoremstyle{definition}

\newcommand{\R}{{\ensuremath{\mathbb{R}}}}

\newcommand{\e}{\mathrm{e}}

\newcommand{\syst}[1]{%
  \left\{
  \begin{array}{l}
    #1
  \end{array}
  \right.
  \kern-\nulldelimiterspace
}

\makeatletter
\renewcommand{\maketitle}{
\begin{center}
    {\LARGE \bfseries \@title \par}

    \vspace{0.8em}

    {\large \@author \par}

    \vspace{1em}

    \rule{\textwidth}{0.5pt}
\end{center}

\vspace{1.5em}
}
\makeatother

\title{HOT-POT: Optimal Transport for Sparse Stereo Matching}

\author{
\large
Antonin Clerc$^{1,2}$,
Michael Quellmalz$^{1, 3}$,
Moritz Piening$^{1}$,
Philipp Flotho$^{4, 5}$,
Gregor Kornhardt$^{1}$,\\
Gabriele Steidl$^{1}$
\\[1ex]
\small
$^{1}$ Institute of Mathematics, Technische Universität Berlin, Germany
\\
$^{2}$ Univ. Bordeaux, CNRS, Bordeaux INP, IMB,
UMR 5251, F-33400 Talence, France
\\
$^{3}$ Department of Mathematics, National Tsing Hua University,
Hsinchu, Taiwan
\\
$^{4}$ Chair for Clinical Bioinformatics, Saarland Informatics Campus,
Saarland University, Germany
\\
$^{5}$ Optical Neuroimaging Unit, Okinawa Institute of Science and Technology, Japan
\\[1ex]
\texttt{antonin.clerc@math.u-bordeaux.fr},
\texttt{quellmalz@math.tu-berlin.de},
\texttt{piening@math.tu-berlin.de},
\texttt{Philipp.Flotho@uni-saarland.de},
\texttt{kornhardt@math.tu-berlin.de},
\texttt{steidl@math.tu-berlin.de}
\\[2ex]
Accepted for publication in the Journal of Mathematical Imaging and Vision.
\\
DOI:
\href{https://doi.org/10.1007/s10851-026-01336-3}
{10.1007/s10851-026-01336-3}
}

\begin{document}

\twocolumn[
\maketitle

\centering

\begin{minipage}{0.80\textwidth}

\begin{abstract}
Stereo vision between images faces a range of challenges, including
occlusions, motion, and camera distortions, across applications in
autonomous driving, robotics, and face analysis. Due to parameter
sensitivity, further complications arise for stereo matching with
sparse features, such as facial landmarks.
To overcome this ill-posedness and enable unsupervised sparse matching,
we consider line constraints of the camera geometry from an optimal
transport (OT) viewpoint.
Formulating camera-projected points as (half)lines, we propose the use
of the classical epipolar distance as well as a 3D ray distance to
quantify matching quality.
Employing these distances as a cost function of a (partial) OT problem,
we arrive at efficiently solvable assignment problems.
Moreover, we extend our approach to unsupervised object matching by
formulating it as a hierarchical OT problem.
The resulting algorithms allow for efficient feature and object
matching, as demonstrated in our numerical experiments.
Here, we focus on applications in facial analysis, where we aim to
match distinct landmarking conventions.
\end{abstract}

\vspace{2em}

\end{minipage}
]

\section{Introduction}
Identification of features across views and inference of depth information via \textbf{stereo matching} is a core technique in computer vision, allowing for 3D reconstructions from multiple 2D views \cite{Liu2023}. 
It enables obstacle detection in robotics and autonomous driving, 
surface defect detection in industrial applications, and
head reconstruction in facial analysis \cite{Kok_Rajendran_2019, Liu2023, prince2012computer, scharstein1999view}. 
However, practical algorithms need to overcome a variety of real-world challenges. 

Even for single-modality camera systems, the identification of view-invariant features is generally hindered by radiometrically distorted pixel brightness, 
depth changes near object boundaries, and partial occlusions. 
These issues become more pronounced for cross-modal systems like RGB-thermal.
In particular, traditional stereo matching methods relying on the comparison of pixel values
\cite{CFNSS2015, dalal2005histograms,fsian2022comparisonstereomatchingalgorithms} become inapplicable under these conditions. Therefore, suitable cross-spectral approaches are mostly based on pretrained neural networks \cite{liu2022multi_cross_descr, kim2016dasc_cross_descr, pingera2012CSstereo_cross_desc,NNstereo2025} and local feature descriptors  \cite{liang2022deep_cross_learn,zhi2018deep_cross_learn}. 
Notably, most modern feature extractors are {dense}  \cite{guo2023unsupervised,li2024local}, producing one feature per pixel. In contrast, we focus on \textbf{sparse features} located at image keypoints \cite{schauwecker2012new}.

Beyond the detection of feature descriptors, their cross-view matching is inherently ill-posed since it is highly sensitive to detector noise and occlusions \cite{mikolajczyk2005performance,pingera2012CSstereo_cross_desc}, especially if only the location and no additional information, such as color, is available.
Moreover, this instability increases dramatically for {sparse features} as considered in this work.
In this setting, key challenges arise from significant differences between descriptors across views in terms of their number and their locations.
Such limitations necessitate the modification of nearest-neighbor matching to ensure robustness~\cite{schauwecker2012new}.

{\color{black}
Recent deep learning methods have addressed several related stereo and correspondence tasks. Dense and cross-spectral stereo networks usually predict pixelwise disparity or depth from image pairs by learning feature representations or cost aggregation mechanisms. Examples include H-Net \cite{huang2022h}, local-structure-guided stereo matching \cite{li2024local}, deep cross-spectral stereo matching \cite{liang2022deep_cross_learn}, and material-aware cross-spectral stereo matching~\cite{zhi2018deep_cross_learn}. Learned image matchers such as SuperGlue \cite{sarlin2020superglue} and LoFTR \cite{sun2021loftr} instead estimate sparse or semi-dense image correspondences, with SuperGlue being particularly related through its differentiable OT assignment layer. Structure-aware correspondence learning for relative pose estimation addresses another adjacent problem, learning object-level keypoints and correspondences for pose estimation under limited overlap \cite{chen2025structure}. In contrast, we assume that sparse, possibly modality-dependent landmarks are already given and solve a training-free assignment problem based on camera calibration and point geometry. While most deep learning-based methods can exploit image appearance and learned priors, our OT/POT/HOT formulation remains applicable when only sparse points are available.
}

\textbf{Optimal transport (OT)} provides a robust relaxation of nearest-neighbor matching \cite{COTFNT,Villani2003}, 
relaxing the optimal assignment problem to a linear program over probabilistic assignments.
Due to the availability of efficient solvers {via Sinkhorn's method \cite{Cuturi2013,flamary2021pot,Knight2008} or slicing \cite{BonRabPeyPfi15,KolNadSimBadRoh19,nguyen2025introduction,QueBeiSte23,QueBueSte2024}} as well as extensions to partial assignments \cite{Chapel2020,chizat2018unbalanced,sejourne2023unbalanced}, 
the OT framework has become a popular tool for comparing and matching point clouds \cite{leroy2021pix2point,li2021selfpointflow,puy2020flot,shen2021accurate}.
This has further led to applications in stereo vision with dense features by matching along image rows using OT \cite{frameworkstereoOT} and by using an OT-based module in the stereo matching model H-Net \cite{huang2022h},
which allows the network to focus on feature mismatches regarding the so-called epipolar constraints of the camera geometry. 
Extending this approach, we derive distances between feature keypoints from these epipolar constraints to perform sparse stereo matching using OT assignments.
Beyond feature matching via \textbf{classical OT}, 
we integrate the resulting cost functions into \textbf{hierarchical OT} \cite{alvarezmelis2020geometricdatasetdistancesoptimal, delon2020wassersteintypedistancespacegaussian, 
nguyen2025lightspeed,Schmitzer2013HOT}
to enable object matching based on sparse keypoints.

As a particular application, the main motivation for this study is a \textbf{cross-modal stereo vision} setting similar to \cite{flotho2021multimodal}. We are given facial image pairs captured simultaneously by a conventional RGB camera and a thermal (long-wave infrared) camera, whose intrinsic parameters and locations are known. 
The main goal is the matching of facial features obtained by pretrained modality-dependent feature trackers or \textbf{landmarkers} \cite{bulat2017far,kemelmacher2016megaface,Tang2017,wood2021fake,wu2019facial}.
These landmarkers 
discretize the underlying 3D facial geometry and may yield distinct sets of points for different modalities \cite{deng2019menpo,flotho2025t}, see Figure~\ref{fig:thermal}.
Therefore, the two sets of 2D points are subject to modality-specific noise and may differ in cardinality. 
Furthermore, for two stereo images containing faces of many persons, we want to identify reliably which face in the RGB image belongs to which face in the thermal image.

\begin{figure}[!ht]
    \centering

    \begin{subfigure}[b]{0.32\linewidth}
        \centering
        \includegraphics[width=\linewidth,
          trim=10mm 50mm 10mm 10mm,
          clip
        ]{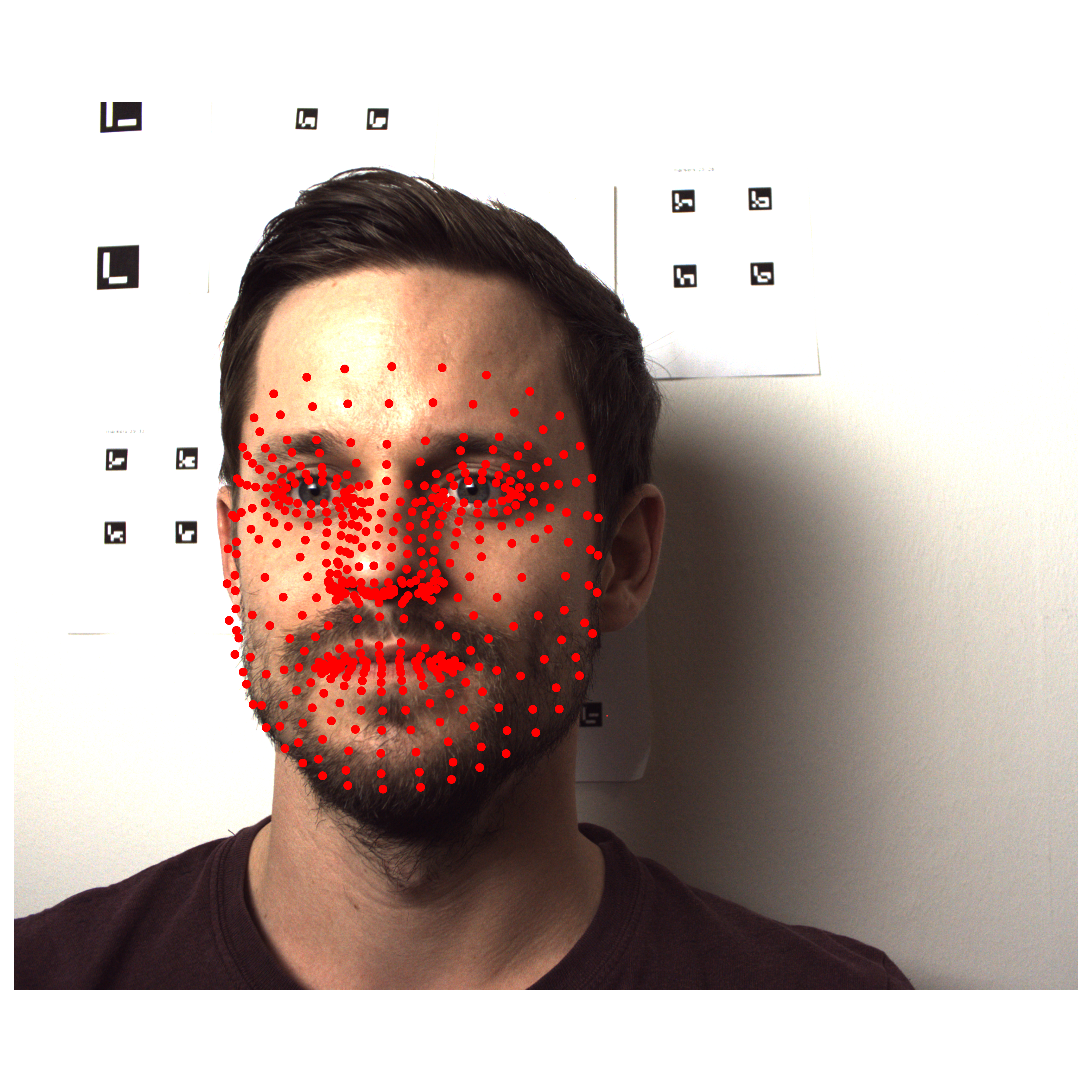}
        \caption*{1st RGB}
        \label{fig:real_left_rgb}
    \end{subfigure}
    \hfill
    \begin{subfigure}[b]{0.32\linewidth}
        \centering
        \includegraphics[width=\linewidth,
          trim=10mm 50mm 10mm 10mm,
          clip
        ]{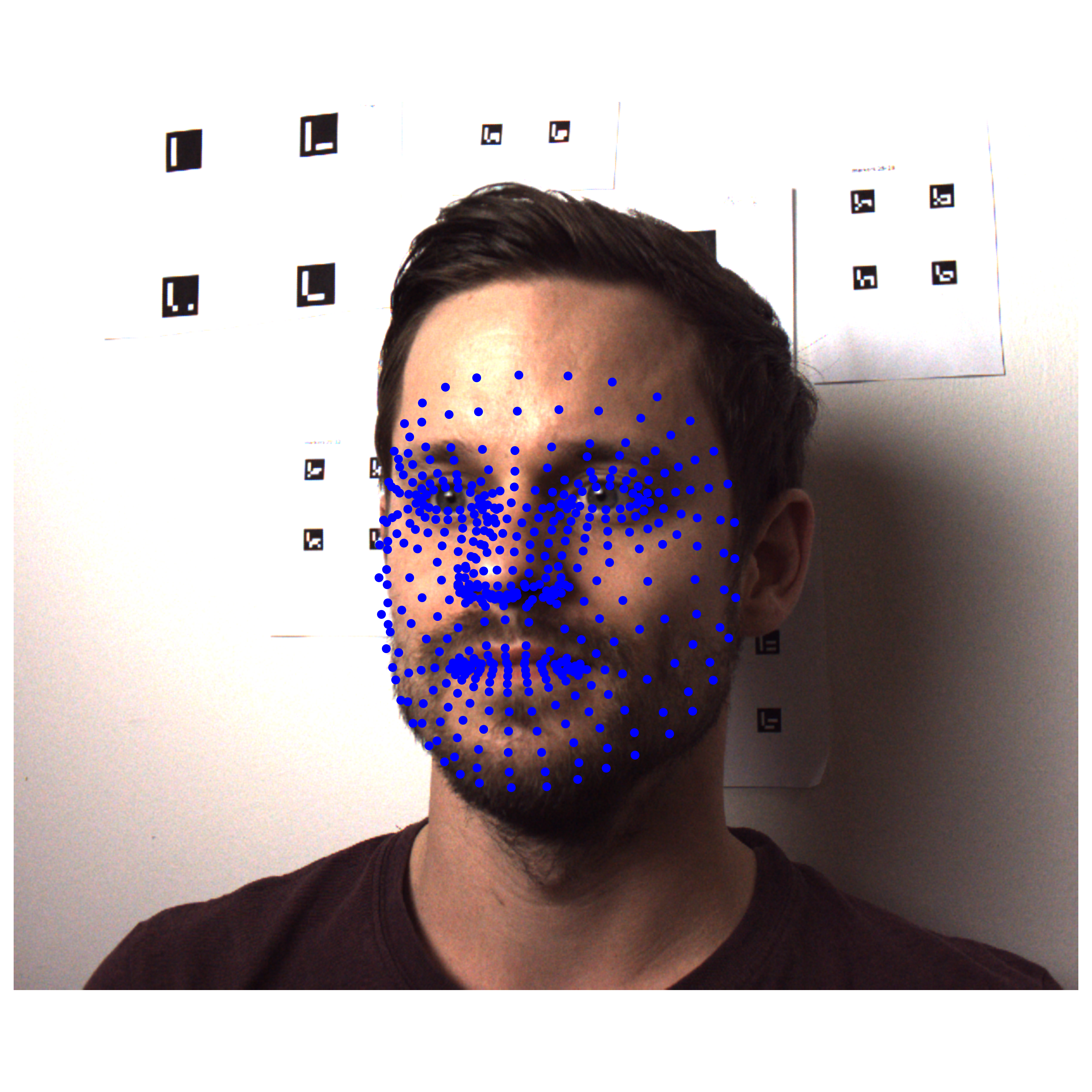}
        \caption*{2nd RGB}
        \label{fig:real_right_rgb}
    \end{subfigure}
    \hfill
    \begin{subfigure}[b]{0.32\linewidth}
        \centering
        \includegraphics[width=\linewidth,
          trim=30mm 50mm 10mm 10mm,
          clip
        ]{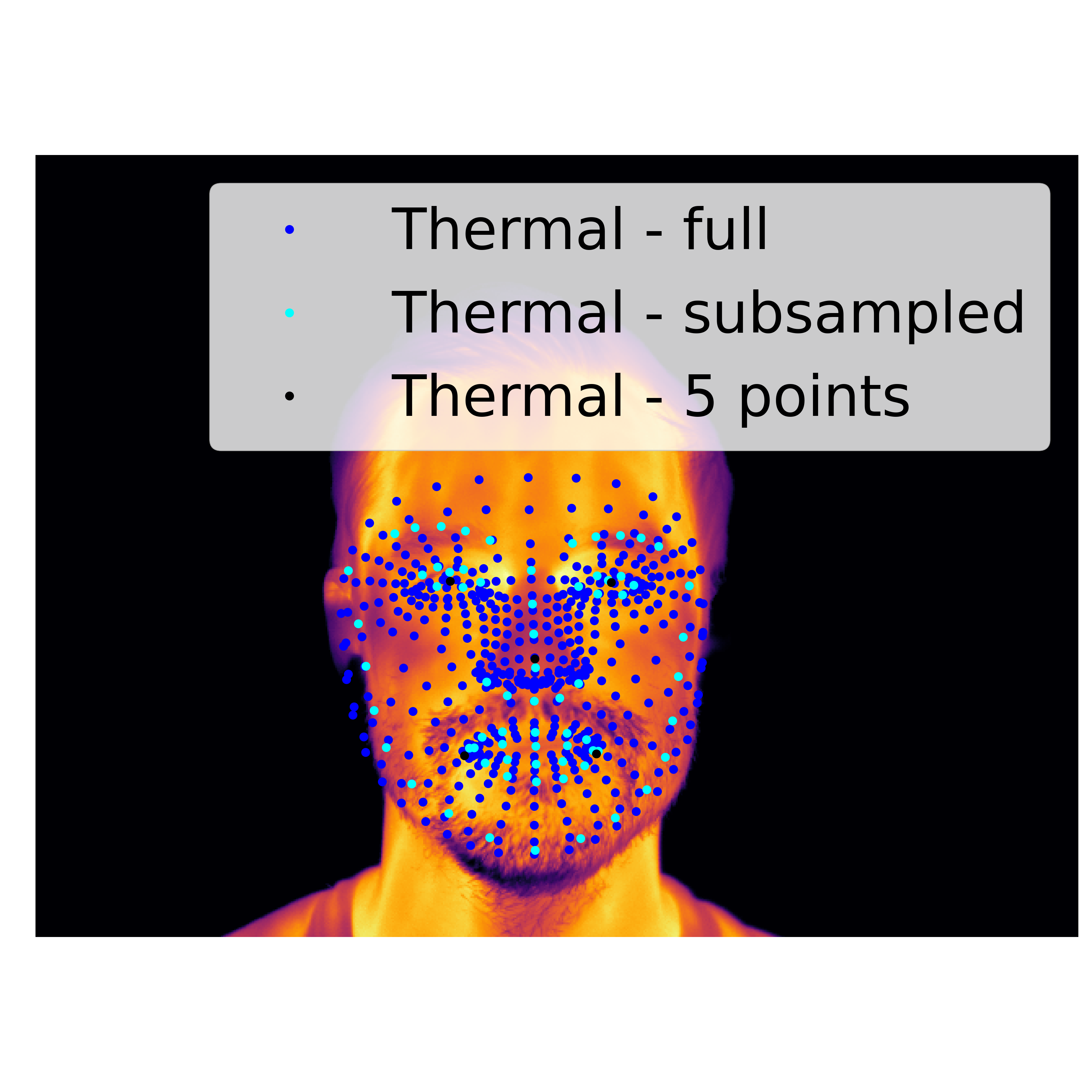}
        \caption*{\scriptsize  Thermal}
        \label{fig:real_right_thermal}
    \end{subfigure}
    \caption{RGB (left, middle) and thermal (right) images with facial landmarks. 
    Landmarks conventions vary in terms of size and locations, as illustrated by the thermal image. Images were originally reported in \cite{flotho2023lagrangian}.}
        \label{fig:thermal}
\end{figure}

We model our setup via an unknown 3D point cloud that is projected onto both camera planes, yielding two distinct 2D point clouds, which are considered as our measurements. 
Our practical goal is then i)~to perform an accurate point-to-point matching across camera planes and ii)~to establish correspondences between entire objects, each containing multiple points.
In practice, these 2D point clouds are obtained via given feature tracking algorithms, such as a landmarker.
Within the setup of facial landmarking, we pursue the goal of cross-modal
landmark-to-landmark and face-to-face matchings.

Our main contributions are as follows:

\begin{itemize}
    \item We formulate the stereo matching problem as an instance of an OT problem \cite{COTFNT,Villani2003} with cost functions tailored to the geometric setting.
    
    \item We propose a novel cost function, called 3D ray distance, based on the distance between 3D half-lines (rays) originating from the camera centers. Numerical experiments show its advantage over the classical epipolar distance \cite{Hartley2004, prince2012computer}.
    
    \item To handle matching between entire objects (e.g., faces) represented as unordered sets of points, we introduce a \textbf{hierarchical optimal transport} (HOT) formulation. Numerical studies demonstrate that HOT yields stable correspondences.
\end{itemize}

Our method addresses several of the classical challenges in stereo vision. The proposed cost function enhances the robustness of point-level matching by reducing sensitivity to noise and ambiguity, and by limiting the set of admissible matches. This proves to be particularly effective for occlusions and cross-modality scenarios. The use of OT offers an efficient formulation for the matching problem. Finally, the HOT framework extends our algorithm from point-to-point to object-to-object matching.

The paper is organized as follows.
Section~\ref{sec:model} covers the epipolar geometry and introduces our new matching cost. 
Section~\ref{sec:OT} introduces basic concepts of OT required to understand the framework, and outlines the specific OT problems we consider. 
Section~\ref{sec:algo} focuses on algorithmic considerations, including numerical implementation, error computation, and evaluation metrics to assess the quality of the matchings. Section~\ref{sec:Numerical_Experiments} presents numerical results on various simulated and real-world datasets  both for OT and HOT formulations.
Conclusions are drawn in Section~\ref{sec:conclusions}.

\section{Distances of Projected Points} \label{sec:model}
In this section, we propose two different  \enquote{distances} between projected points in the camera planes
which aim to preserve their true, unknown  distances in 3D, namely
\begin{itemize}
\item the 3D ray distance and
\item the epipolar distance.
\end{itemize}
For a detailed treatment of the underlying epipolar geometry and camera models, we refer to \cite{hartley2004multiple,prince2012computer,steger2022multi}.
Considering two cameras observing the same 3D point $w\in \R^3$, we
can express the points measured by the cameras in homogeneous coordinates with the third component fixed to one, i.e.,
as elements of the projective space
$$
\mathbb P^2\coloneqq \{(x_1,x_2,1) \mid (x_1,x_2)\in\R^2\}.
$$ 
We choose the coordinate system such that the left camera is centered at $(0,0,0)^\top$ and is imaging along the positive third coordinate.
By $K_l,\, K_r \in \R^{3\times3}$ we denote the \textbf{intrinsic matrix} of the left and right camera, respectively. These are upper triangular matrices with positive diagonal entries, which encode internal characteristics such as focal length and principal point of the camera. The \textbf{extrinsic parameters} describe the relative orientation and position between the two cameras, modeled by a rotation matrix $R\in \text{SO}(3)$ and a translation vector $t \in \R^3\setminus\{0\}$.
Then the projections  $x,y \in \mathbb P^2$ of a point $w \in \R^3$ 
onto the left and right camera plane are given by
\begin{equation} \label{eq:model}
\lambda_l x = K_l w 
\quad \text{and} \quad
\lambda_r y = K_r (Rw + t),
\end{equation}
respectively,
where $x,y \in \mathbb P^2$ and $\lambda_l,\lambda_r >0$ denote the third component of $K_lw$ and $K_r w$, respectively.
We assume that $w$ is located in front of both cameras, meaning that 
\begin{equation} \label{eq:front}
\langle w, e_3\rangle>0 \quad \text{and} \quad \langle Rw+t,e_3\rangle>0.
\end{equation}
The configuration is shown in Figure \ref{fig:epipolar_geometry}.
Let 
$$
\mathcal W \coloneqq \{w \in \R^3: \langle w, e_3\rangle>0,\;\langle Rw+t,e_3\rangle>0\}.
$$

\begin{figure}
    \centering
    \includegraphics[width=\linewidth]{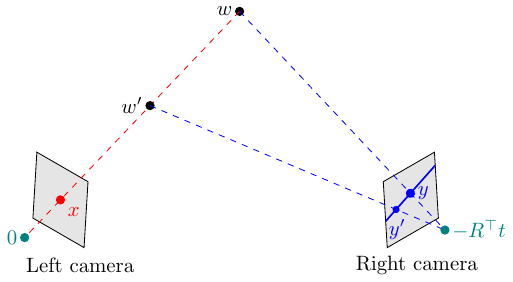}
    \caption{3D point $w$ observed by two cameras, specifically at $x$ in the left and $y$ in the right camera.
    The epipolar line (solid blue) in the right camera corresponding to~$x$ consists of all points $y'$ that originate from 3D points $w'$ that are projected to $x$ in the left camera.
    In black are the focal points $0$ and $-R^\top t$.
    }
    \label{fig:epipolar_geometry}
\end{figure}

\begin{remark} \label{rem:model_general}
    In a more general model, where each camera has a rotation $R_l,R_r\in \rm{SO}(3)$ and translation $t_r,t_l \in \R^3\setminus\{0\}$, the projections are given by
    \begin{align} \label{eq:model_general}
    \lambda_l  x
    &= 
    K_l (R_l \tilde w +t_l)
    ,\\
    \lambda_r  y
    &= 
    K_r(R_r \tilde w + t_r).
    \end{align}
   Then the substitution $\tilde w=R_l^\top (w - t_l)$, and setting $R=R_r R_l^\top$ and $t=t_r-R_r t_l$ yields \eqref{eq:model}.
\end{remark}

\subsection{3D Ray Distance}\label{sec:ray}
Setting 
$$
\tilde x \coloneqq K_l^{-1} x 
\quad \text{and} \quad
\tilde y \coloneqq K_l^{-1} y,
$$
we obtain by \eqref{eq:model} that
\begin{equation} \label{eq:lines}
        w = \lambda_l \tilde x 
        \quad \text{and} \quad
        w = \lambda_r R^\top \tilde y - R^\top t.
\end{equation} 
The right-hand side of each of the two equations represents a line in 
$\R^3$ 
of the form 
\begin{equation} \label{eq:line_3d}
\begin{aligned}
&\mathcal L_x
=
\{\lambda_x r_x + s_x: \lambda_x\in\R\},
\\ 
&\mathcal L_y
=
\{\lambda_y r_y + s_y: \lambda_y\in\R\},
\end{aligned}
\end{equation}
where $r_x=\tilde x$, $s_x=0$, $r_y=R^\top\tilde y$, and
$s_y=-R^\top t$.
The minimal distance between these lines, $\text{d}(\mathcal L_x,\mathcal L_y)$, is given by
\begin{equation} 
 \left\{
\begin{array}{lcll}
\frac{|\langle r_x\times r_y, s_y-s_x\rangle|}{\|r_x\times r_y\|}
&=&
\frac{|\langle \tilde x\times R^\top\tilde y, R^\top t\rangle|}{\|\tilde x\times R^\top\tilde y\|}
&\text{if }r_x\times r_y\neq0,\\[1ex]
\frac{\|r_x\times (s_y-s_x)\|}{\|r_x\|}
&=&
\frac{\|\tilde x\times R^\top t\|}{\|\tilde x\|}
&\text{otherwise } 
\end{array}
\right.
\end{equation}
and could serve as a possible distance between $x$ and $y$. Note that $r_x \times r_y = 0$ if and only if 
$\mathcal L_x \parallel \mathcal L_y$.
However, this distance is computed on the entire space $\mathbb{R}^3$, but
we are only interested in intersections or minimal distances for points that are visible to both cameras.
Indeed, the lines \eqref{eq:line_3d} may intersect or attain their minimal distance behind the cameras. In such cases, although the distance between the lines is small, the corresponding point is not observable by the cameras. Therefore, such matches should be avoided. 

To this end, consider the case  $r_x\times r_y\neq0$. 
Then there exists a unique shortest line segment connecting $\mathcal L_x$ and $\mathcal L_y$. It intersects the respective lines in 
\begin{align}
\label{eq:closest_points}
&b_x = s_x + \frac{\langle s_x - s_y, {n}_x\rangle}{\langle {r}_x, {\color{black}{n}_y}\rangle} {r}_x,\\
&b_y = s_y -\frac{\langle s_x - s_y, {n}_y\rangle}{\langle {r}_y, {\color{black}{n}_x}\rangle} {r}_y, 
\end{align} 
where $n_x \coloneqq r_x \times(r_x\times r_y)$ and
$n_y \coloneqq r_y \times(r_x\times r_y)$.
Both points $b_x$, $b_y$ should fulfill \eqref{eq:front}.
If $b_x$ or $b_y$ lies behind one of the cameras, the ray distance is usually smallest between the focal points of the two cameras.
Hence, we use the distance between the focal points, namely
$\|R^\top t-0\| = \| t \|$.
Thus, we define \textbf{ray distance}
$d^\mathrm{ray}: \mathbb P^2 \to \R_{\ge 0}$ 
by
\begin{equation} \label{eq:line_corrected_distance}
\small
d^\mathrm{ray}(x,y)
\coloneqq
\left\{
\begin{array}{ll}
\frac{|\langle \tilde x \times R^\top \tilde y, R^\top t \rangle|}{\|\tilde x \times R^\top \tilde y\|}
& \text{if } r_x\times r_y \neq 0,\, 
\\
&
\text{and } b_x,b_y \in \mathcal W,
\\
\frac{\|\tilde x \times R^\top t\|}{\|\tilde x\|} & \text{if } r_x\times r_y = 0, \\[10pt]
\| t \| & \text{otherwise}.
\end{array}
\right.
\end{equation}

The so-defined ray distance is not a metric, 
because, in general (depending on the camera parameters), 
$d^\mathrm{ray}$ is not symmetric and $d^\mathrm{ray}(x,x)$ may be non-zero.
However, we have the following properties.

\begin{proposition}
\label{prop:ray_distance_pseudo_definite_property}
    If $w\in\mathcal W$ and $x,y\in\mathbb P^2$ are given by \eqref{eq:model},
    then $d^\mathrm{ray}(x,y)=0$.
    Conversely, if $x,y\in\mathbb P^2$ and $d^\mathrm{ray}(x,y)=0$, 
    there exists $w\in\R^3$ such that \eqref{eq:model} holds for some $\lambda_l,\lambda_r\in\R$. 
    If additionally $r_x \times r_y \neq 0$, then $w$ is uniquely determined.
\end{proposition}
\begin{proof}
Let $w$ satisfy \eqref{eq:front} and $x,y$ be the projections of $w$ via \eqref{eq:model}. 
By construction, 
$w$ is on both lines $\mathcal L_x$ and $\mathcal L_y$.
If the lines are not identical, i.e., if $r_x\times r_y\neq0$, we have $w=b_x=b_y$.
By \eqref{eq:front}, we are in the first case of the definition of $d^\mathrm{ray}$ and hence we have $d^\mathrm{ray}(x,y)=0$.
Otherwise, if the lines are identical, i.e., if $r_x\times r_y=0$, their distance is zero and so $d^\mathrm{ray}(x,y)=0$.

Conversely, let $d^\mathrm{ray}(x,y)=0$.
We note that $\|t\|$ does not vanish by assumption.
If $0=r_1\times r_2 = \tilde x\times R^\top\tilde y$, we have $\tilde x\times R^\top t=0$.
Hence, $\tilde x$, $R^\top\tilde y$, and $R^\top t$ are all located one line through the origin and neither does vanish.
Hence, \eqref{eq:lines}, which is equivalent to \eqref{eq:model}, is fulfilled for $w=\alpha \tilde x$ with any $\alpha\neq0$ and some $\lambda_l,\lambda_r\in\R$.
Otherwise, the lines intersect in one point $b_x=b_y$, and hence \eqref{eq:lines} is fulfilled with $w=b_x=b_y$.
In this case, $w$ satisfies~\eqref{eq:front}.
\end{proof}

\begin{remark}[Depth-regularized Ray Distance]
{\color{black}
While Proposition~\ref{prop:ray_distance_pseudo_definite_property} justifies the use of the ray distance theoretically, its stability depends on the geometric configuration of the rays. In particular,  \eqref{eq:line_corrected_distance} involves the quantity $\|\tilde x \times R^\top \tilde y\|$ in the denominator. If the angle between the rays is bounded away from zero, this quantity is uniformly bounded below, and the distance behaves stably under perturbations of the input data.
In contrast, if the rays are nearly parallel, $\|\tilde x \times R^\top \tilde y\| \to 0$, leading to an ill-conditioned expression. Then, small perturbations in the inputs may induce large variations in the estimated distance and correspond to intersections at very large or infinite depth. This instability is inherent to the geometry of skew lines rather than to the specific formulation of $d^\mathrm{ray}$.
}
To improve stability,
a common remedy is the use of regularization to include prior knowledge
\cite{benning2018modern}.
In most practical scenarios, we have prior information about the possible depth of the objects, e.g., if the camera is located in a room, the objects cannot be farther away than the walls.
Motivated by this, we consider lower and upper soft thresholds $\gamma_1<\gamma_2$ on the depth and some regularization parameter $\beta\ge0$.
Then, we modify our distance \eqref{eq:line_corrected_distance} to introduce the \textbf{depth-regularized ray distance} $d^\mathrm{reg}_{\beta,\gamma}(x,y)$ defined as
\begin{equation} \label{eq:corrected_distance_line_z}
    d^\mathrm{ray}(x,y)
    +
    \beta
    \begin{cases}
        (b-\gamma_1)^2 & {\rm{ if }} \; b<\gamma_1,\\
        0 & {\rm{ if }} \; b\in[\gamma_1,\gamma_2],\\
        (b-\gamma_2)^2 & {\rm{ if }}  \; b>\gamma_2,\\
    \end{cases}
\end{equation}
where 
$b \coloneqq \frac12 \langle b_x + b_y, e_3\rangle$ is the third coordinate of the midpoint of $b_x,b_y$ given in~\eqref{eq:closest_points}. 
{\color{black}This regularization explicitly penalizes unrealistic depth values, thereby preventing degenerate configurations associated with quasi-parallel rays from producing spurious matches at infinite distance.}
\end{remark}

\begin{remark}[Invariance to Rotations] \label{rem:rot}
    The ray distance is invariant to rotations of the camera with the same focal point.
    More specifically, taking a different right camera with parameters $R'\in SO(3)$ and $t'\in\R^3$ and the same focal point $R^\top t = (R')^\top t'$.
    Then the ray distance $d^\mathrm{ray}(x,y)$ 
    coincides with the ray distance between $x$ and $y'$
    with respect to the other camera, where $ y'$  is 
    the normalized projection of $w\in\R^3$ 
    fulfilling $w=\lambda_{r}'(R')^\top \tilde y' - (R')^\top t'$, cf.\ \eqref{eq:lines}.
\end{remark}

\subsection{Epipolar Distance}\label{sec:epi}
The second distance arises from epipolar (half)lines, which are often used in the
literature, see, e.g. \cite{Darmon2021,Haming2007,Pollefeys1999}.
Substituting $w = \lambda_l \tilde{x}$ into the second equation in \eqref{eq:lines}
gives
\begin{equation}
\lambda_r \tilde{y} = \lambda_l R \tilde{x} + t.
\end{equation}
Taking the cross product of both sides with the translation vector $t = (t_1,t_2,t_3)^\top$ results in
\begin{equation} \label{eq:epi-cross}
t \times (\lambda_r \tilde{y}) = t \times (\lambda_l R \tilde{x}),
\end{equation}
which can be rewritten as
\begin{equation} 
\lambda_r T \tilde{y} = \lambda_l T R \tilde{x}, 
\qquad
T \coloneqq
\left(
\begin{array}{rrr}
0 & -t_3 & t_2 \\
t_3 & 0 & -t_1 \\
-t_2 & t_1 & 0
\end{array}
\right).
\end{equation}
Since the cross product is perpendicular to its generating vectors, we obtain by taking the inner product with $\tilde y$
the \textbf{epipolar constraint} 
\begin{equation} \label{eq:fundamental_constraint}
y^\top F x = 0, \qquad  F \coloneqq K_r^{-\top} T R K_l^{-1},
\end{equation}
where $F$ is called \textbf{fundamental matrix}.
The \textbf{epipolar line with respect to} $x \in \mathbb P^2$ in the right camera plane is given by
\begin{equation} \label{epiline}
\{(s_1,s_2)^\top: \langle Fx, s \rangle = 0 \}, \; s := (s_1,s_2,1)^\top.
\end{equation}
Geometrically, the epipolar line is the projection of the line  $\mathcal L_x$ in \eqref{eq:line_3d}
to the right camera plane, see Figure~\ref{fig:epipolar_geometry}.
Then the distance
of a point $y\in\mathbb P^2$ in the right image to the epipolar line \eqref{epiline}  of $x\in\mathbb P^2$ is given by
\begin{equation} \label{eq:di}
   d^\mathrm{epi}_r(x, y) 
   \coloneqq
   \begin{cases}\displaystyle
    \frac{|\langle F x,y\rangle|}{\|P_2{F x}\|} 
   &\quad \text{if} \quad
   P_2{F x}\neq 0,
   \\ 
      |(Fx)_3| &\quad \text{if} \quad
   P_2{F x}= 0,
   \end{cases}
\end{equation}
where $P_2{(x_1,x_2,x_3)}\coloneqq(x_1,x_2)$
and $(Fx)_3$ denotes the third component of the vector.
The second case is motivated by the fact that if $P_2Fx=0$, 
we have $Fx=(0,0,(Fx)_3)^\top$
and hence
$\langle Fx,y\rangle =(Fx)_3$ for all $y\in\mathbb P^2$. 
Then the epipolar line \eqref{epiline} degenerates to either the plane $\mathbb P^2$ if $(Fx)_3=0$,
or the empty set if $(Fx)_3\neq0$.
By the next remark, the case $P_2{F x}= 0$ cannot appear
for points in front of the camera and if the cameras do not see each other.

\begin{remark} \label{rem:special_case}
Assume that $P_2Fx=0$. This implies
$Fx=\alpha e_3$ with $\alpha=(Fx)_3$ and since $K_r$ is upper triangular further
\begin{align*}
F x &= K_r^{-\top} T R K_l^{-1} x = \alpha e_3,
\\
T R K_l^{-1} x &= \alpha (K_r)_{3,3} e_3,\\
T R w &= t \times R w = \lambda_l \alpha (K_r)_{3,3} e_3. 
\end{align*}
If $\alpha\neq0$, this means that
$$
\langle Rw,e_3\rangle 
=\langle t,e_3\rangle= \langle R^\top t, R^\top e_3 \rangle = 0,
$$
and consequently $\langle Rw+t,e_3\rangle=0$  which contradicts \eqref{eq:front}.
Geometrically, all such $w$ are in the plane  with normal direction $R^\top e_3$ 
through the focal point $-R^\top t$ of the second camera. 
If $\alpha=0$, then $w$ is a multiple of $-R^\top t$ 
and lies therefore on the line between the two focal points, see Figure~\ref{fig:epipolar_geometry}.
Then the epipolar line degenerates to a single point.
In practice, this case does usually not occur as it would mean that the second camera is visible in the first camera image. 
\end{remark}

Similarly, we can consider the distance
of a point $x\in\mathbb P^2$ in the left image to the epipolar line  of $y\in\mathbb P^2$ to obtain the \textbf{epipolar distance} in the left image
\begin{equation} \label{eq:dj}
   d^\mathrm{epi}_l(x, y) 
   \coloneqq
   \begin{cases}\displaystyle
      \frac{|\langle x, F^\top y\rangle|}{\|P_2{F^\top y}\|}
   &\quad \text{if} \quad 
   P_2{F^\top y}\neq 0,
   \\ 
   |(F^\top y)_3| &\quad \text{if} \quad
   P_2{F^\top y}= 0.
   \end{cases}
\end{equation}
Averaging over the epipolar distances in the left and right image, we get
the final \textbf{epipolar distance}
\begin{equation} \label{eq:epipolar_distance}
d^{\mathrm{epi}}(x,y) 
\coloneqq \frac{
d^{\mathrm{epi}}_l(x,y) +
d^{\mathrm{epi}}_r(x,y) }{2},
\end{equation}
in particular
\begin{align} 
d^{\mathrm{epi}}(x,y) 
=& 
\frac{|\langle Fx,y \rangle |}2
\left( \frac{1}{\|P_2{Fx}\|} + \frac{1}{\|P_2{F^\top y}\|}\right)\\
&\text{if}\quad P_2{F x}\neq0\text{ and }P_2{F^\top y}\neq0.
\end{align}

\begin{remark}[Epipolar rays]

Instead of the line~\eqref{eq:line_3d}, we could project the ray located in front of the camera to the image plane,
leading to an {epipolar ray}, cf.\ \cite{Darmon2021,Haming2007,Pollefeys1999}.
The condition~\eqref{eq:front}, which means a point is in front of both cameras, is equivalent to $\lambda_1,\lambda_2>0$ in~\eqref{eq:model}.
By~\eqref{eq:epi-cross}, we have
$$
\tfrac{\lambda_2}{\lambda_1}\, t \times \tilde{y} = t \times R \tilde{x},
$$
which holds for $\lambda_1\lambda_2>0$ if and only if
\begin{equation} \label{eq:half-epipolar-eq}
\langle t \times \tilde{y}, t \times R \tilde{x} \rangle > 0
\end{equation}
or $t\times\tilde y = t\times R\tilde x=0$.
We define for fixed $x\in\mathbb P^2$ the {epipolar ray} $H_x$ in the right image by
\begin{equation} \label{eq:half-epipolar}
\{y\in\mathbb P^2 : \langle y,Fx\rangle=0,\, \langle t\times K_r^{-1}y,t\times RK_l^{-1}x\rangle >0\}.
\end{equation}

We show that in most practical scenarios, where both cameras depict a similar region of the 3D space,
the epipolar rays and epipolar lines coincide inside the images.
More specifically, we assume the camera image is a square
$$
I_{a} \coloneqq 
[-a,a]^2 \times \{1\} \subset \mathbb P^2
,\qquad a>0.
$$
The {epipole} $y_\e\in\mathbb P^2$ is the projection of the focal point $w=0$ of the left camera to the right image
given by
\begin{equation} \label{eq:epipole}
\lambda_l x_\e = -K_l R^\top t.
\end{equation}
We assume that $y_\e$ is not visible in the right camera, i.e., $y_\e\notin I_{a}$,
and that $w\in\R^3$ is visible by both cameras with the projections $x\in \mathbb P^2$ and $y\in I_{a}$.
Then 
\begin{equation} \label{eq:half-line-intersection}
H_x\cap I_{a}
=
\{y\in I_{a} : \langle y,Fx\rangle =0\}.
\end{equation}
This can be seen as follows.
By definition, the left-hand side is a subset of the right. 
We show that \eqref{eq:half-epipolar-eq} is fulfilled for all $y\in I_{a}$ with $\langle Fx,y\rangle=0$.
The function $y\mapsto \langle t\times\tilde y,t\times R\tilde x\rangle$ is affine-linear, 
hence its zero set is affine-linear in $\mathbb P^2$
and it contains the epipole $y_\e$.
Therefore, $H_x$ is either a ray starting at the epipole $y_\e$ or empty if \eqref{eq:half-epipolar-eq} vanishes on the whole epipolar line. 
Hence, the intersection $H_x\cap I_{a}$ is either $I_{a,b}\cap \{y\in\mathbb P^2: \langle y,Fx\rangle=0\}$ or empty.
The latter cannot hold, as the set contains the projection of~$w$.
\end{remark}

\section{Optimal Transport} \label{sec:OT}
We recall basic notions of discrete OT and its partial version as well as hierarchical OT, cf.\ \cite{COTFNT}. 
While we formulate the OT problem for discrete measures on the projective plane $\mathbb P^2$, the theory in this section also applies to continuous probability measures on general manifolds.

\subsection{Point Matching via Optimal Transport}
Given two sets
\begin{align} \label{eq:point_clouds}
X&=\{x^1,\dots,x^N\}\in(\mathbb P^2)^N ,\\
Y&=\{y^1,\dots,y^M\}\in(\mathbb P^2)^M,
\end{align}
we are interested in the optimal transport costs between the empirical measures
\begin{equation} \label{eq:measures}
\mu = \frac1N \sum_{i=1}^{N} \delta_{x^i} 
\quad \text{and}\quad 
\nu = \frac1M\sum_{i=1}^{M} \delta_{y^i}, 
\end{equation}
where $\delta_x$ is the point measure at $x$,
defined by the discrete \textbf{OT distance}
\begin{equation}\label{eq:OT_one_to_one}
\begin{aligned}
&\text{OT}_c(X,Y) \coloneqq \min_{\Pi \ge 0} \, 
\langle C,\Pi \rangle\\
 \text{subject to} \quad
&\Pi \mathbf{1}_M = \frac1N \mathbf{1}_N, \;
\Pi^\top \mathbf{1}_N= \frac1M \mathbf{1}_M,
\end{aligned}
\end{equation}
where $c\colon\mathbb P^2\times\mathbb P^2\to[0,\infty[$ is a cost or ``distance function'' on $\mathbb P^2$,
$$
C \coloneqq \left( c(x^i,y^j) \right)_{i,j=1}^{N,M}
\quad \text{and}\quad 
\Pi \coloneqq \left(\pi_{i,j} \right)_{i,j=1}^{N,M}.
$$
The matrix $\Pi \in \R_{\geq 0}^{N \times M}$ is called a transport plan.
In our applications, we will deal with the ``distances''
$d^{\mathrm{ray}}$, $d^\mathrm{reg}_{\beta,\gamma}$ and
$d^{\mathrm{epi}}$
from the previous section.
If $N=M$, there exists an optimal solution $\Pi$ such that  $N \Pi$
is a permutation matrix associated with a permutation $\sigma\in\operatorname{Perm}(N)$
and 
$
\text{OT}_c(X,Y) = \sum_{i=1}^N c(x^i, y^{\sigma(i)} ) $, see \cite[Prop~2.1]{COTFNT}.

However, in a practical stereo matching application, 
object occlusions or modality-dependent feature trackers
might violate the assumption of the balanced setup $N = M$.
Clearly, for $N \not = M$, a transport plan cannot be realized by a permutation matrix.
Therefore, we consider the more general \textbf{partial optimal transport} (POT) problem \cite{caffarelli2010free, Chapel2020}, which explicitly restricts the total transported mass 
and is defined as
\begin{align} \label{eq:POT_matching}
&\operatorname{POT}_c(X,Y)
\coloneqq  \min_{\Pi \ge 0} \langle C, \Pi \rangle\\
& \text{subject to} \quad {\mathbf{1}_N^\top \, \Pi \, \mathbf{1}_M = m,}\\
&\quad 
\Pi \mathbf{1}_M \leq \frac1N {\mathbf{1}_N} 
, \quad 
\Pi^\top \mathbf{1}_N \leq \frac1M { \mathbf{1}_M},
\end{align}
{for a mass constraint $m\in[0,1]$.}
While it simplifies to \eqref{eq:OT_one_to_one} for $N=M$ {and $m=1$}, it allows excess mass to be discarded for $N \neq M$. 
{Consequently, POT is especially well-suited for computing a partial correspondence between point clouds resulting from incompatible landmarkers.}
If $m={\min\{N, M\}}/{\max\{N, M\}}$,
there exists an optimal solution $\Pi/m \in \{0, 1\}^{N \times M}$ such that $\Pi$ is associated with a `partial' permutation between the smaller set and an equal-sized subset of the larger set, see \cite[Thm. 4.1]{bai2023sliced}.

Beyond partial matching, one might alternatively replace POT with the closely related and more general \textbf{unbalanced optimal transport} formulation \cite{beier2023unbalanced,sejourne2023unbalanced}, where one regularizes the transport problem with respect to a fixed probability divergence, see \cite[Appendix~A]{stein2025} for examples.
{However, the latter has two regularization parameters that need to be chosen appropriately.}

\subsection{Object Matching via Hierarchical Optimal Transport}
\textbf{Hierarchical optimal transport} (HOT) refines the constraints of the original OT problem by introducing additional labels.
In the following, let $[N] \coloneqq \{1,\ldots,N\}$.
Now the goal is  to match full objects each consisting of many points.
Given $N$ source objects $X_i$ and $M$ target objects $Y_j$, let
\begin{align} \label{eq:objects}
&\boldsymbol{X}=\{X_1,\dots,X_N\}
\quad \text{and} \quad
\boldsymbol{Y}=\{Y_1,\dots,Y_M\}
\\
&\text{with} \quad
X_i = \{x_i^1, \dots, x_i^{N_i}\} \in(\mathbb P^2)^{N_i}, \;
i\in [N],
\\
&\hphantom{with} \quad
Y_j = \{y_j^{1}, \dots, y_j^{M_j} \} \in(\mathbb P^2)^{M_j},
\;
j\in [M].
\end{align}
We use a two-step hierarchical strategy, inspired by hierarchical Wasserstein distance formulations \cite{alvarezmelis2020geometricdatasetdistancesoptimal, delon2020wassersteintypedistancespacegaussian, yurochkin2019hierarchical}. 
In the case $N=M$ and a Euclidean cost function, our hierarchical OT formulation computes a discretized version of the so-called Wasserstein over Wasserstein \cite{bonet2025flowing,piening2025slicing_wow} or
the mixture Wasserstein distance \cite{delon2020wassersteintypedistancespacegaussian,piening2025slicing}.
\\[1ex]
\textbf{Step 1 (Local pointwise matching between objects):}
For each pair  $(X_i, Y_j)$, $i\in[N]$,\ $j\in [M]$, we solve the POT problem \eqref{eq:POT_matching} to obtain
\begin{equation} \label{eq:hot_step1}
    c^{\mathrm{obj}}(X_i,Y_j) \coloneqq 
    \mathrm{POT}_c(X_i, Y_j).
\end{equation}
\textbf{Step 2 (Global matching):}
Once the object-to-object cost $c^{\mathrm{obj}}$ has been computed, we solve 
a second POT problem 
at the object level: 
\begin{equation} \label{eq:hot_step2}
    \mathrm{POT}_{c^{\mathrm{obj}} } (\boldsymbol{X}, \boldsymbol{Y}).
\end{equation}

The hierarchical matching procedure is summarized in Algorithm~\ref{algo:hierarchical_matching}. 
In the case of a balanced matching using \eqref{eq:OT_one_to_one}, we refer to this procedure as \textbf{HOT}. If we employ the POT formulation \eqref{eq:POT_matching}, we use the name \textbf{HOT-POT}.

\begin{algorithm}[!ht]
\caption{Hierarchical Object Matching} \label{algo:hierarchical_matching}
 \KwIn{
    Sets of $N$ source objects $\boldsymbol{X}$ and $M$ target objects $\boldsymbol{Y}$\\
    Cost function $c$ between points in $\mathbb P^2$\\
     }
\KwOut{Binary matching matrix $ \Pi^\mathrm{obj}$}

\ForEach{$i \in [N], \,  j \in [M]$}
    {Compute pointwise POT cost
    \hspace{1em} 
    $c^{\mathrm{obj}}(X_i,Y_j) = \mathrm{POT}_c (X_i, Y_j)$}

Compute transport plan $\Pi^{\mathrm{obj}}$ minimizing
\hspace{1em} $\mathrm{POT}_{c^{\mathrm{obj}}}(\boldsymbol{X}, \boldsymbol{Y})$

\Return{$\Pi^{\mathrm{obj}}$}
\end{algorithm}
{\color{black} 
Algorithm~\ref{algo:hierarchical_matching} defines a two-level optimization procedure composed of standard OT (or POT) problems. Using the distances defined in Sections~\ref{sec:ray} and \ref{sec:epi}, the first step is a well-posed convex optimization problem admitting at least one optimal solution. The resulting matching cost is therefore well defined, and the second-level problem is also a well-posed optimization problem admitting at least one solution. When solved numerically (e.g., via Sinkhorn-type or linear programming methods), these subproblems converge to their corresponding optimal or entropically regularized transport plans.}

\paragraph*{Recovering a global pointwise map}
From HOT or HOT-POT, we may again compute a global point matching. Let 
$N_{\mathrm{tot}}=\sum_{i=1}^N N_i$ and $M_{\mathrm{tot}}=\sum_{j=1}^M M_j$,
and denote by ${ \Pi^{i,j} \in \R^{N_i \times M_j}}$ a pointwise POT plan between 
$X_i$ and $Y_j$, and by $\Pi^{\mathrm{obj}}$ the object-level plan.  
The global plan 
$\Pi^{\mathrm{glob}}\in\mathbb{R}_{\ge0}^{N_{\mathrm{tot}}\times M_{\mathrm{tot}}}$ 
is obtained by embedding each local plan ${ \Pi^{i,j}}$ into its
corresponding block and scaling it by the transported mass $\Pi^{\mathrm{obj}}_{i,j}$:
We set
\begin{equation}\label{eq:global_plan}
\Pi^{\mathrm{glob}}_{(i,r),(j,s)}
\coloneqq
\Pi^{\mathrm{obj}}_{i,j}\,
  \Pi^{i,j}_{r, s},
\end{equation}
as the mass transported between $x_r^i$ and $y_s^j$.
Our resulting pointwise plan $\Pi^{\mathrm{glob}}$ becomes binary if all input plans are binary.

\section{Algorithmic Considerations} \label{sec:algo}

In this section, we discuss several aspects of the implementation of our algorithms as well as error metrics. We will deal both with pointwise  and objectwise matching.
\subsection{OT Algorithms}
\textbf{OT.}
For computing $\text{OT}_c(X,Y)$  with $c \in \{d^{\text{epi}},d^{\text{ray}},d^\mathrm{reg}_{\beta,\gamma}\}$,
we apply the {Earth Mover's Distance} algorithm \cite{Bonneel2011} implemented in the PythonOT library \cite{flamary2021pot}.
It returns a permutation matrix, but does not guarantee uniqueness and can be sensitive to input order.
\\[1ex]
\textbf{POT.}
For partial OT, use the solver \texttt{ot.partial.partial\_wasserstein}  \cite{caffarelli2010free,Chapel2020} from PythonOT, which implements a relaxed optimal transport formulation, with a partial mass constraint $m={\min\{N, M\}}/{\max\{N, M\}}$.
\\[1ex]
\textbf{HOT/HOT-POT.}
For our hierarchical matching procedure, we utilize Algorithm~\ref{algo:hierarchical_matching} based on the aforementioned PythonOT solvers.

\begin{remark}[Projecting onto Binary Matrices]
While we know that our OT and POT problems can be solved by scaled (partial) permutation matrices, see \cite[Prop. 2.1]{COTFNT}  and \cite[Thm. 4.1]{bai2023sliced},
practical solvers relying on continuous relaxations may return soft transport plans.
In that case, we project them to binary matrices by assigning each point to the maximizing index only if at least half of the mass is concentrated there, and discarding it otherwise.
\end{remark}

\begin{remark}[Naive Matching]\label{rem:naive}
As baseline for comparing the performance of the OT and POT algorithms 
we use a \textbf{naive matching} procedure between two sets of points $X = \{x^1, \ldots, x^N\}$ and $Y = \{y^1, \ldots, y^M\}$.
We find the smallest value within the cost matrix $C_{ij} = c(x^i, y^j)$, take the respective indices $i^*,j^*$ for our matching, and remove the row $i^*$ and the column $j^*$.
As the smallest value might be non-unique, we take the first occurrence.
We repeat this procedure until there is no row or column is left and obtain in total $\min(N,M)$ matches.
\end{remark}

\subsection{Evaluation Criteria}
\label{sec:matching_rate}
\paragraph*{Comparing with Ground-Truth Matching}
Given a point cloud imaged by two different cameras, we can directly calculate the \textbf{pointwise mismatch rate} as the number of incorrectly matched point pairs divided by the total number of point pairs:
\[
\frac{\#\text{incorrect point matches}}{\min\{N,\,M\}}.
\]
For our parameter $m={\min\{N, M\}}/{\max\{N, M\}}$, the total number of matches is $\min\{N,M\}$.

If we have multiple objects and each is described by a point cloud in the left and a point cloud in the right camera, we can perform object matching via HOT or HOT-POT.
We calculate the \textbf{objectwise mismatch rate} as the number of incorrectly matched object pairs divided by the total number of object pairs:
\[
\frac{\#\text{incorrect object matches}}{\min\{N,M\}}.
\]
In our multi-modal landmarking system, each object is a single face that is described by i) the landmarks of an RGB tracker on the left and ii) a thermal landmarker on the right camera. 
For example, if the RGB tracker detects 5 faces and the thermal one detects 4, the total number of object pairs becomes 4. If our algorithm matches 2 face pairs correctly, the object mismatch ratio becomes $2/4=0.5$.

\paragraph*{Evaluating Matching based on 3D Reconstruction.} \label{sec:unsupervised_eval}
In some applications, we might be more interested in reconstructing the true 3D point cloud than in recovering the exact point-to-point matching. 
In other scenarios, we might want to evaluate the pointwise
matching quality in the absence of a ground-truth correspondence, e.g., due to occlusions.
In this setting, an object is represented by two incompatible 3D point clouds, one visible on the left camera and one on the right. 

In both cases, we can evaluate this setup based on the 3D reconstruction of our 3D point clouds. Given a pair $(x, y) \in \mathbb P^2 \times \mathbb P^2$, 
we can solve \eqref{eq:model} for $w \in \R^3$ if we know the ground truth camera parameters and $r_x \times r_y \neq 0$, see Prop.~\ref{prop:ray_distance_pseudo_definite_property}. 
Thus, we can reconstruct (triangulate) a 3D point cloud and compare the \textbf{(squared) Wasserstein-2 distance} \cite{COTFNT} between some ground truth point cloud and the reconstructed point cloud, i.e., we compute the minimum in  \eqref{eq:OT_one_to_one} with the squared Euclidean distance in $\R^3$ as our cost function $c$ (up to rescaling).

\subsection{Computational Complexity}

{\color{black}
The complexity of algorithms solving OT problems has been extensively studied; we refer to \cite{COTFNT} for a comprehensive overview. We briefly recall here the computational complexity of the methods used in this work.

The Earth Mover’s Distance (EMD) algorithm for OT has a computational complexity of $\mathcal{O}(N^3 \log N)$ \cite{complexityEMD}, where $N$ denotes the number of points in each set. POT admits relaxed formulations with complexity $\mathcal{O}(N M)$ \cite{POTcomplexity}, comparable to the Sinkhorn algorithm for relaxed OT \cite{COTFNT}.
For HOT, where objects are composed of multiple points, the computation involves two levels: pointwise matching within each pair of objects is quadratic in the number of points per object, and matching between objects is quadratic in the number of objects. This yields an overall complexity of
$
\mathcal{O}(
N_{\text{points per object}}^2
\cdot
N_{\text{objects}}^2
).
$
Applying POT directly to all $N_{\text{points per object}} \cdot N_{\text{objects}}$ points leads to the same theoretical complexity.

Execution times for the different algorithms are reported in Table~\ref{tab:exec_time}, obtained on a MacBook Pro equipped with an Apple M1 processor featuring 8 cores and 16 GB of unified memory.
These results depend on implementation details, in particular parallelization. HOT benefits from an inherent parallel structure, as pointwise matchings between object pairs can be computed independently, which can significantly reduce computational time. Similarly, Sinkhorn-based methods (including POT) are highly parallelizable and are therefore often preferred over EMD when reduced computation time is required.

\begin{table}[htb!]
\centering
\small
\begin{tabular}{|l|l|l|l|l|l|l|}
\hline
\textbf{Method} & \textbf{Data} & \textbf{$d^{\text{epi}}$} & \textbf{$d^{\text{ray}}$} & \textbf{$d^{\text{reg}}$} \\
\hline

OT& Sub        &  0.90  & 0.95 &  0.95 \\
(EMD) & Full       &  0.15   & 0.24  &  1.11  \\
\hline
Naive& Sub        &  0.01  & 0.02 & 0.04  \\
Matching& Full       &  3.03   & 3.19   &  3.59   \\
\hline
POT& Sub$\rightarrow$Full  & 0.02  &  0.06  & 0.1    \\
& Full$\rightarrow$Sub  & 0.02  &  0.07  & 0.25    \\
\hline
\multirow{4}{*}{HOT--POT} 
& Sub$\rightarrow$Sub   & 0.01  & 0.03 &   0.18 \\
& Full$\rightarrow$Full  &  0.18 & 0.43  &  1.28 \\
& Sub$\rightarrow$Full  &  0.02  & 0.06  &  0.17   \\
& Full$\rightarrow$Sub  &  0.04 &  0.24  &  0.75   \\
\hline
\end{tabular}
\caption{Execution times in seconds for stereo matching with different algorithms for the synthetic faces dataset, see Section~\ref{sec:Synth_Faces}.}
\label{tab:exec_time}
\end{table}
}

\section{Numerical Experiments} \label{sec:Numerical_Experiments}
In our numerical experiments, 
we first perform experiments on synthetic data to allow for a quantitative comparison between the ray and the epipolar distance. 
Afterwards, we extend our analysis to real-world landmarking data.

\subsection{Synthetic Faces Dataset} \label{sec:Synth_Faces}

\paragraph*{Dataset} 
We use an artificially created dataset of 3D points from four human faces, see Figure~\ref{fig:3d_landmarks_faces}. 
The full dataset contains 1872 points corresponding to four 3D faces, each composed of the 468 landmarks of the MediaPipe canonical face model \cite{lugaresi2019mediapipe}, see Figure~\ref{fig:3d_full}.
The subsampled dataset consists of 65 points per face, {corresponding to 
averaged landmarks of the 3D faces}, see Figure~\ref{fig:3d_sub}. 
For that purpose, the 468 landmarks are downsampled to 65 points by partitioning each region's sorted vertex indices into fixed numbers of chunks and taking the 3D centroid of each chunk. 
We know the ground truth correspondences, meaning that for each point in the left camera, the corresponding point in the right camera is known.
Moreover, we have access to four distinct face labels as employed in our HOT formulation \eqref{eq:hot_step1}--\eqref{eq:hot_step2}.

The camera projections $x$ and $y$ computed via the model \eqref{eq:model_general} are shown in Figure~\ref{fig:2d_landmarks_cameras}. 
In particular, note that the projected faces partially overlap. 
Indeed, such scenarios may appear in practice since some trackers predict landmarks for occluded face regions via interpolating the face geometry \cite{deng2019menpo} or motion in videos \cite{sun2019fab}.

\begin{figure}[!ht]
    \centering
    \begin{subfigure}[b]{0.48\linewidth}
        \centering
        \includegraphics[width=\linewidth]{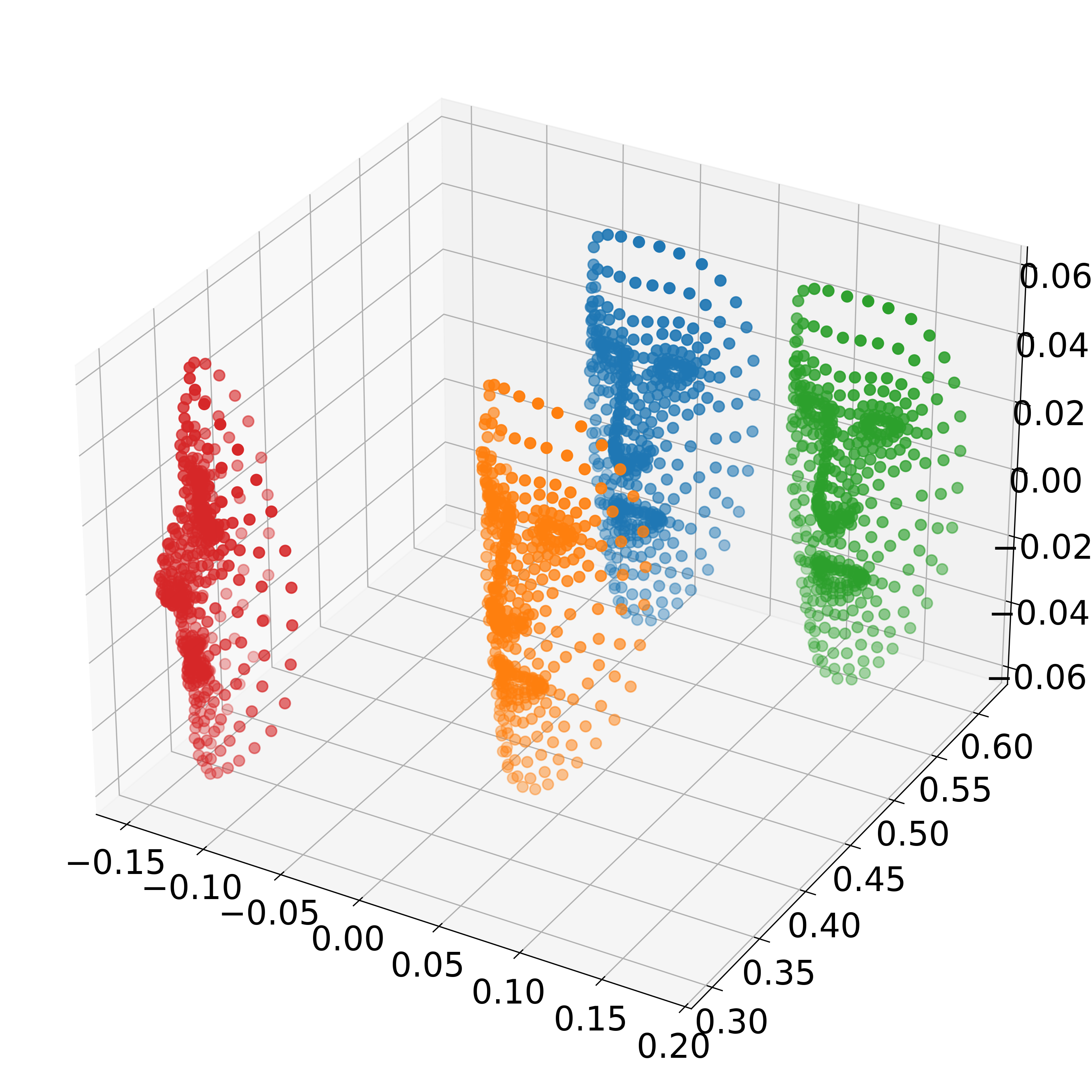}
        \caption{Full (1872 pts)}
        \label{fig:3d_full}
    \end{subfigure}
    \hfill
    \begin{subfigure}[b]{0.48\linewidth}
        \centering
        \includegraphics[width=\linewidth]{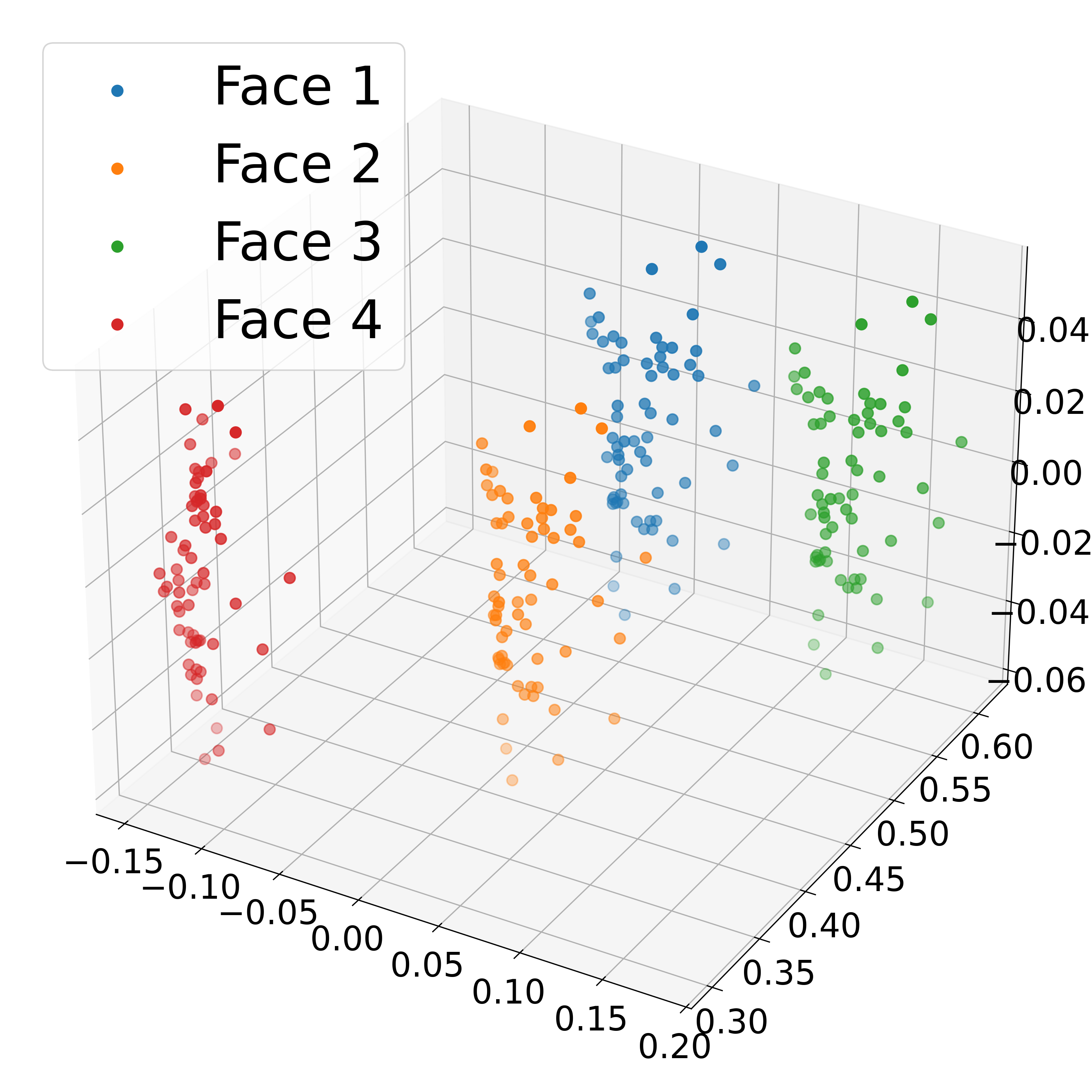}
        \caption{Subsampled (260 pts)}
        \label{fig:3d_sub}
    \end{subfigure}
    \caption{3D landmarks $w$ of four faces.}
    \label{fig:3d_landmarks_faces}
\end{figure}

\begin{figure*}[!ht]
    \centering
    \begin{subfigure}[b]{.48\linewidth}
        \centering
        \includegraphics[width=\linewidth]{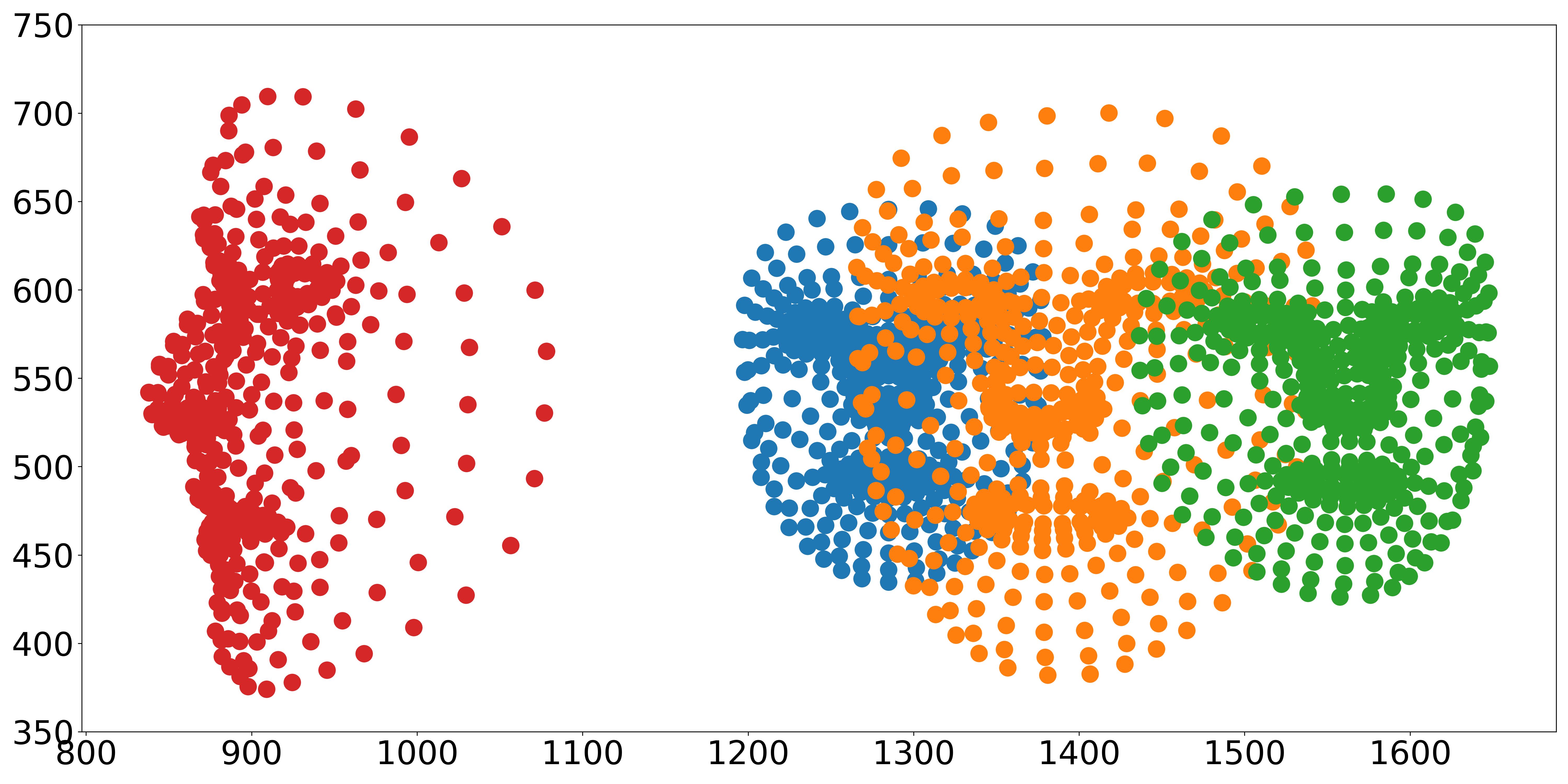}
        \caption{Left camera (points $x_i$)}
        \label{fig:2d_camleft}
    \end{subfigure}
    \hfill
    \begin{subfigure}[b]{.48\linewidth}
        \centering
        \includegraphics[width=\linewidth]{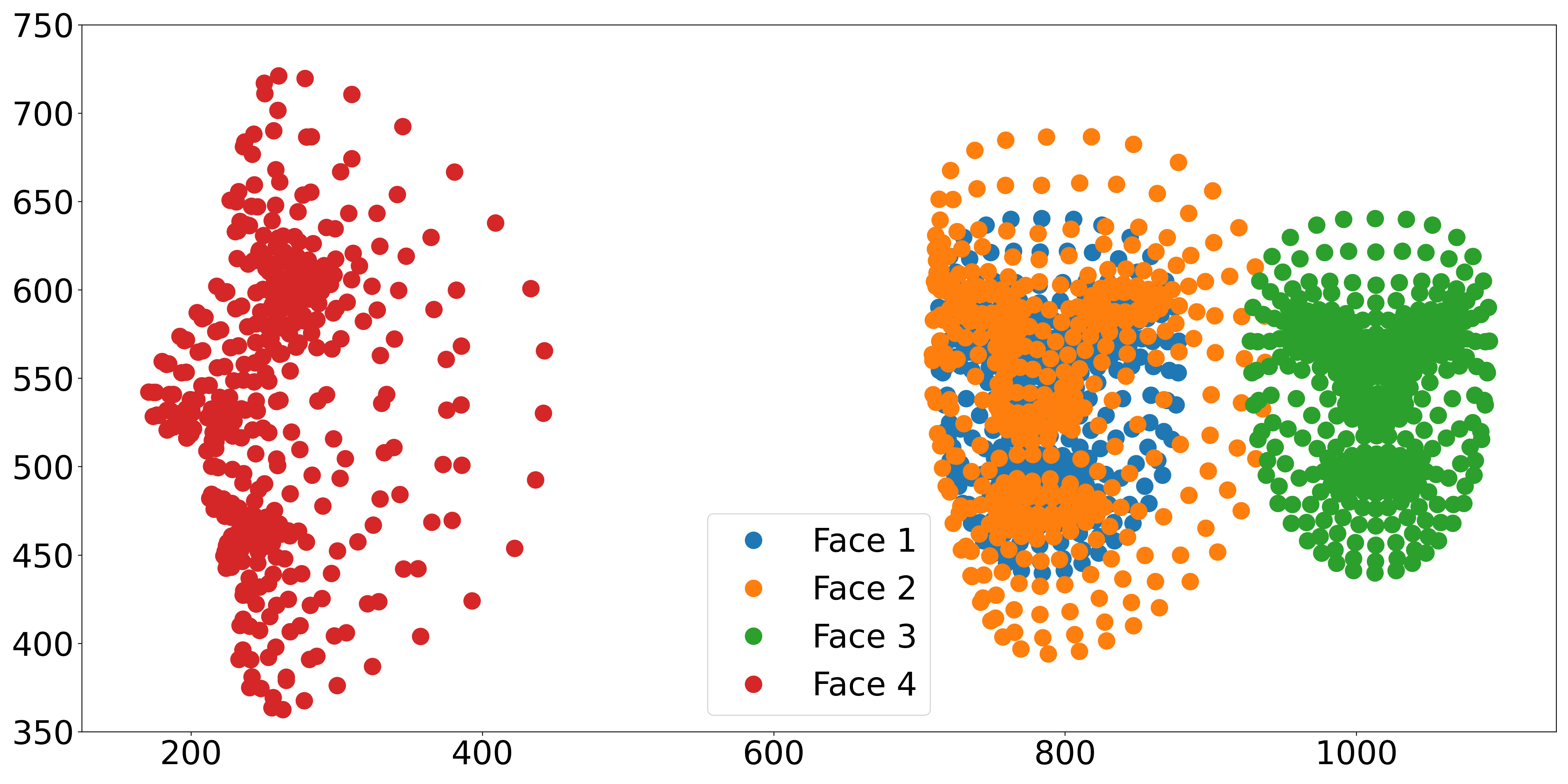}
        \caption{Right camera (points $y_j$)}
        \label{fig:2d_camright}
    \end{subfigure}
    \caption{2D camera images of the four faces dataset from Figure~\ref{fig:3d_full}.}
    \label{fig:2d_landmarks_cameras}
\end{figure*}
\begin{table}[htbp]
    \centering
    \begin{tabular}{|l|c c c|c c c|}
    \hline
    \textbf{$c$} & \multicolumn{3}{c|}{Full ($N=1872$)} & \multicolumn{3}{c|}{Sub ($N=260$)} \\

     & Naive & OT & HOT & Naive & OT & HOT \\
    \hline
    \multirow{2}{*}{$d^{\text{epi}}$}
     & 662 & 681 & 406 & 66 & 64 & 0 \\
     & 35\% & 36\% & 22\% & 25\% & 25\% & 0\% \\
    \hline
    \multirow{2}{*}{$d^{\text{ray}}$}
     & 809 & 284 & 264 & 96 & 0 & 0 \\
     & 43\% & 15\% & 14\% & 37\% & 0\% & 0\% \\
    \hline
    \end{tabular}
    \caption{ Point mismatch counts and ratios for the synthetic faces dataset.}
    \label{tab:OT_Synthfaces}
\end{table}

\paragraph*{Point Matching via OT and HOT with Ground Truth}
We start by performing pointwise matching between the landmarks using the OT formulation \eqref{eq:OT_one_to_one} without any face labels. Table~\ref{tab:OT_Synthfaces} reports the mismatch rates of the OT matching and the naive matching in Remark~\ref{rem:naive} for our two distances. 
Overall, the results with the epipolar distance~\eqref{eq:epipolar_distance} are worse than with the ray distance~\eqref{eq:line_corrected_distance}. 
For the latter, OT leads to a considerable improvement over naive matching,
reducing the number of errors from 809 to 284 for the full dataset, and from 96 to 0 for the subsampled data. 
Geometrically, the poor performance with the epipolar distance can be explained by the fact that points from \textit{face 3} lie very close to epipolar lines \eqref{epiline} corresponding to \textit{face 1}, see Figure~\ref{fig:epipolar_13}. 

If we include the face labels and employ the induced point-to-point HOT transport plan \eqref{eq:global_plan}, we see a drastic improvement for the epipolar-based matching in Table~\ref{tab:OT_Synthfaces}. Nevertheless, the ray distance still leads to better results. 

Figure~\ref{fig:cost_matrix} shows the cost matrices.
For the epipolar distance, the cost nearly vanishes not only along the diagonal, but also along two sub-diagonals corresponding to the association between face 1 and face 3. In contrast, the ray distance can  distinguish faces 1 and 3 more effectively. 

\begin{figure}
    \centering
    \includegraphics[width=\linewidth]{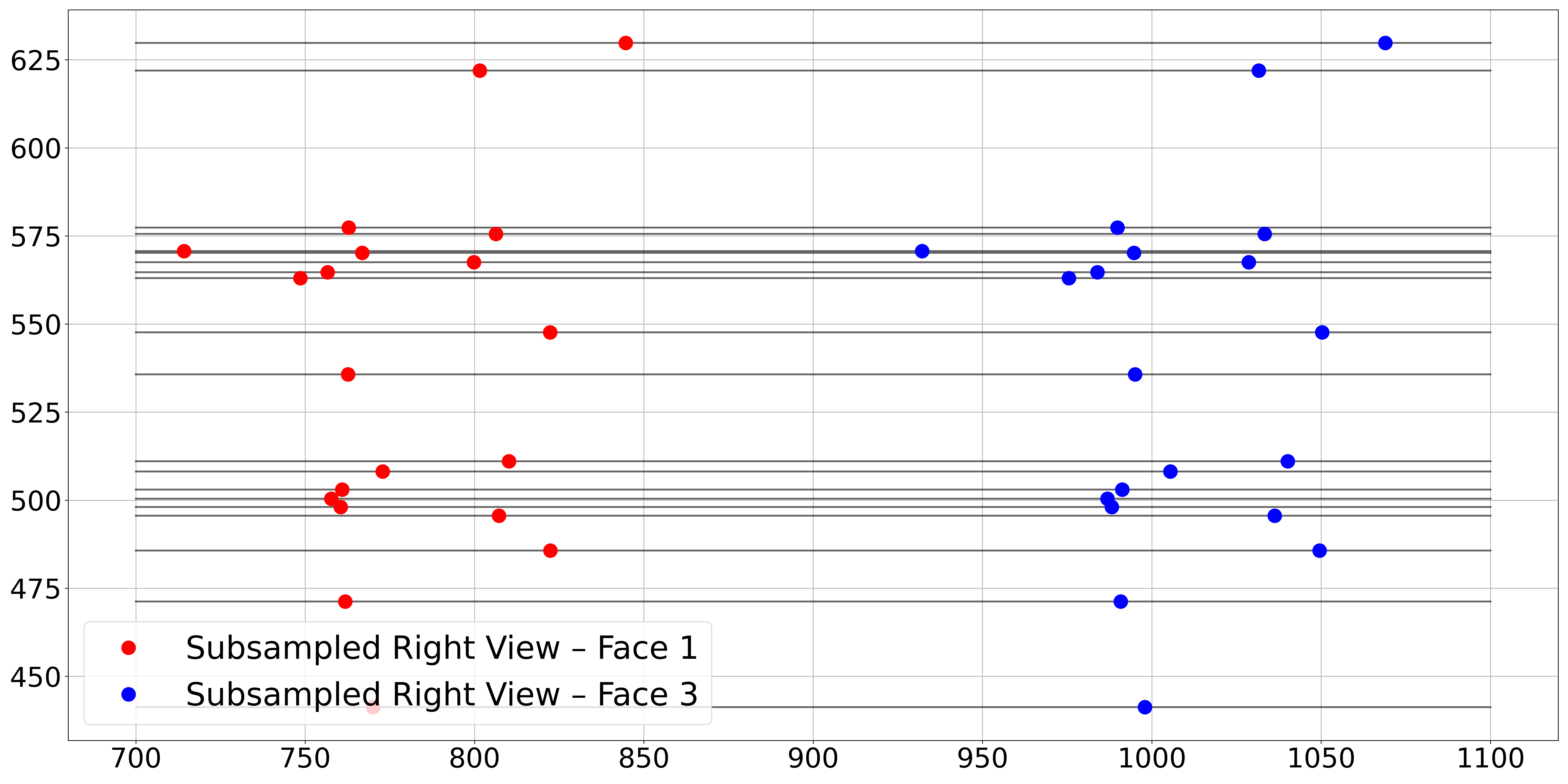}
    \caption{Projection of 20 points from synthetic faces 1 (red) and 3 (blue) onto the right camera and epipolar lines of face 1, illustrating the difficulty of distinguishing faces with $d^{\mathrm{epi}}$ when two points are on a line.}
    \label{fig:epipolar_13}
\end{figure}

\begin{figure}[!htb]
\centering

\begin{subfigure}[t]{0.49\linewidth}
\hfill
    \centering
    \includegraphics[width=\linewidth,trim={0 65 0 65},clip]{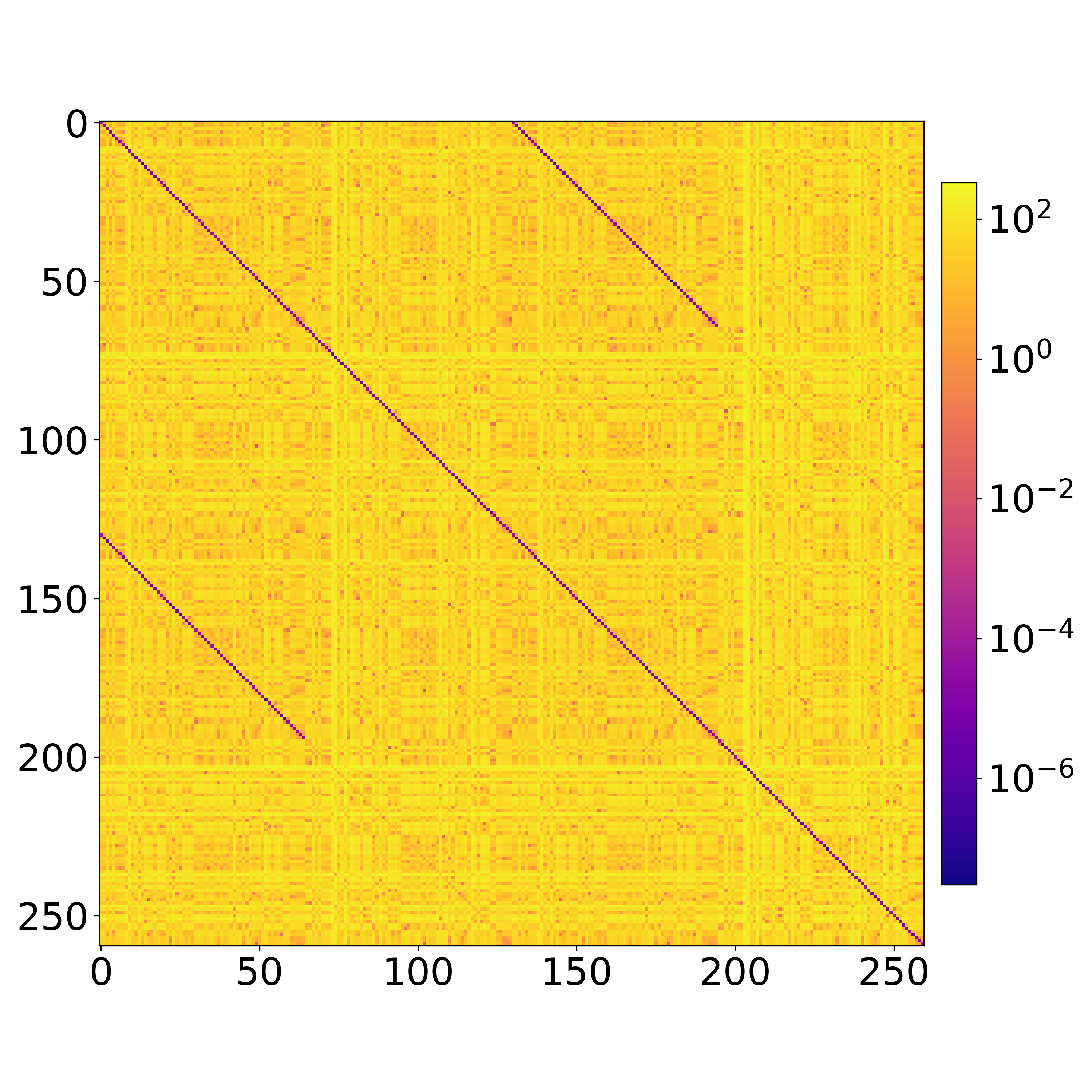}
    \caption{Cost: $d^{\text{epi}}$}
    \label{fig:cost_epipolar}
\end{subfigure}
\hfill
\begin{subfigure}[t]{0.49\linewidth}
    \centering
    \includegraphics[width=\linewidth,trim={0 65 0 65},clip]{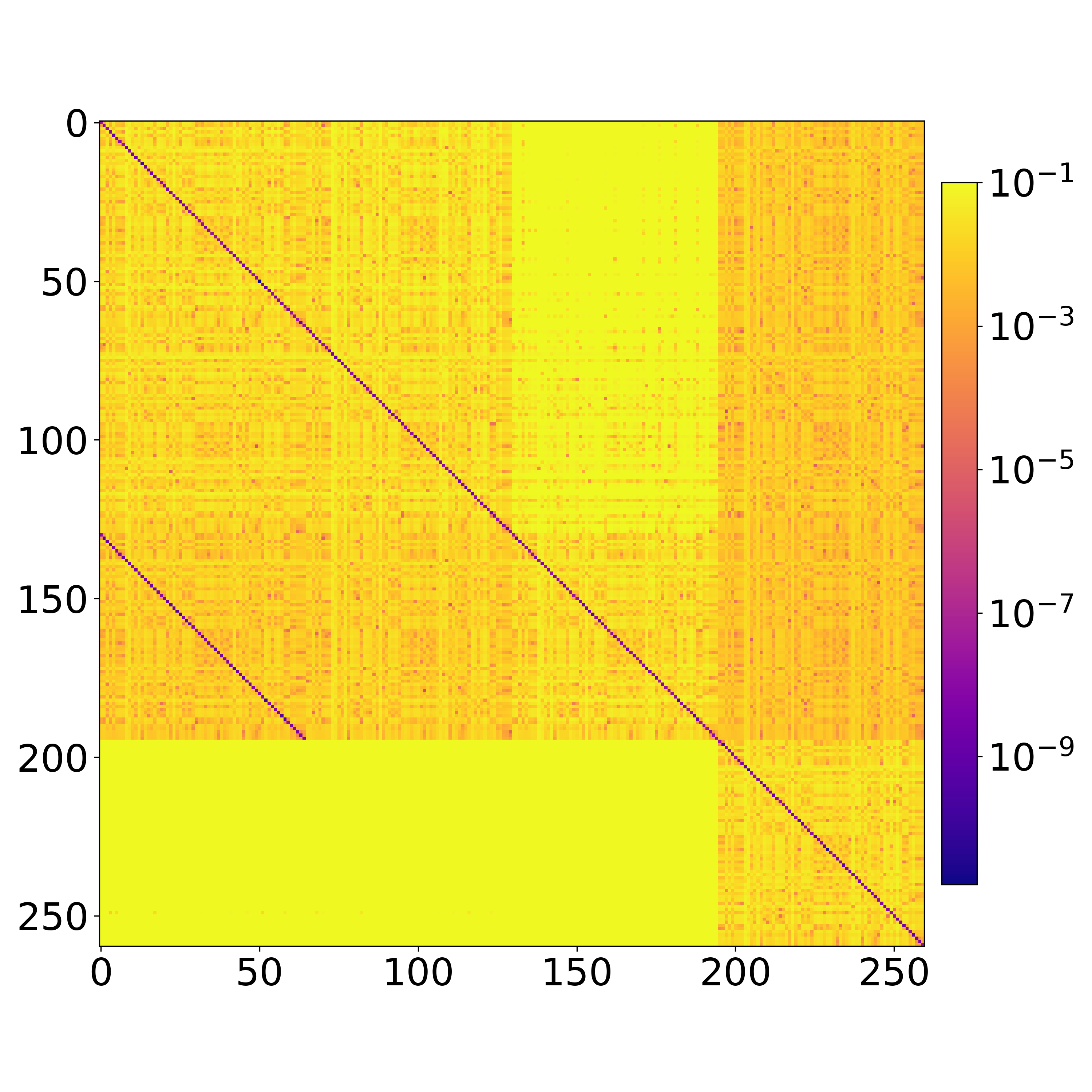}
    \caption{Cost: $d^{\text{ray}}$}
    \label{fig:cost_cline}
\end{subfigure}
\hfill
\caption{Cost matrices $[d(x_i,y_j)]_{ij}$ for the two geometric distances on the subsampled synthetic faces from Figure~\ref{fig:3d_sub}.  
The epipolar cost exhibits low off-diagonal entries linking face~1 and face~3, while the ray cost suppresses these cross-face connections.}
\label{fig:cost_matrix}
\end{figure}

\paragraph*{Point Matching without Ground-Truth via POT}
Next, we investigate the matching of all 1872 on the left camera and the subsampled 260 points on the right camera to assess the partial matching approach.
Using the epipolar distance, 27.3\,\% of all points are matched to the wrong face. For the ray distance, this percentage goes down to 12.3\,\%.
Following Section~\ref{sec:unsupervised_eval}, the squared Wasserstein-2 distance between the reconstructed 3D point clouds 
is reported in Table~\ref{tab:POT_Synthfaces}.

\begin{table}[htbp]
    \setlength{\tabcolsep}{9pt}
    \begin{tabular}{l c c}
    \hline
    \textbf{$c$} & \multicolumn{1}{c}{Full $\rightarrow$ Sub} & \multicolumn{1}{c}{Sub $\rightarrow$ Full} \\
    \hline
    \multirow{1}{*}{$d^{\text{epi}}$}
     & 1.39 & 1.45 \\
    \multirow{1}{*}{$d^{\text{ray}}$}
     & 0.75 & 0.13 \\
    \hline
    \end{tabular}
    
    \caption{Squared Wasserstein-2 distance for the synthetic faces dataset with POT,  lower is better.}
    \label{tab:POT_Synthfaces}
\end{table}

\paragraph*{Face Matching via HOT-POT}
Lastly, we extend the partial matching comparison to object matching via HOT-POT by using the face labels for all points on one camera and the subsampled points on the other camera. 
With this approach, the object-wise mismatch rates are 0\,\% for both distances, i.e., all faces are correctly matched. This holds for both setups, i.e., with the subsampled data on the left and the full data on the right and vice versa.

\subsection{Synthetic Spheres Simulation}
\begin{figure}[!tbh]
    \centering
    \begin{subfigure}[t]{\linewidth}
        \centering
    \includegraphics[width=\linewidth,trim={0 35 10 20},clip]{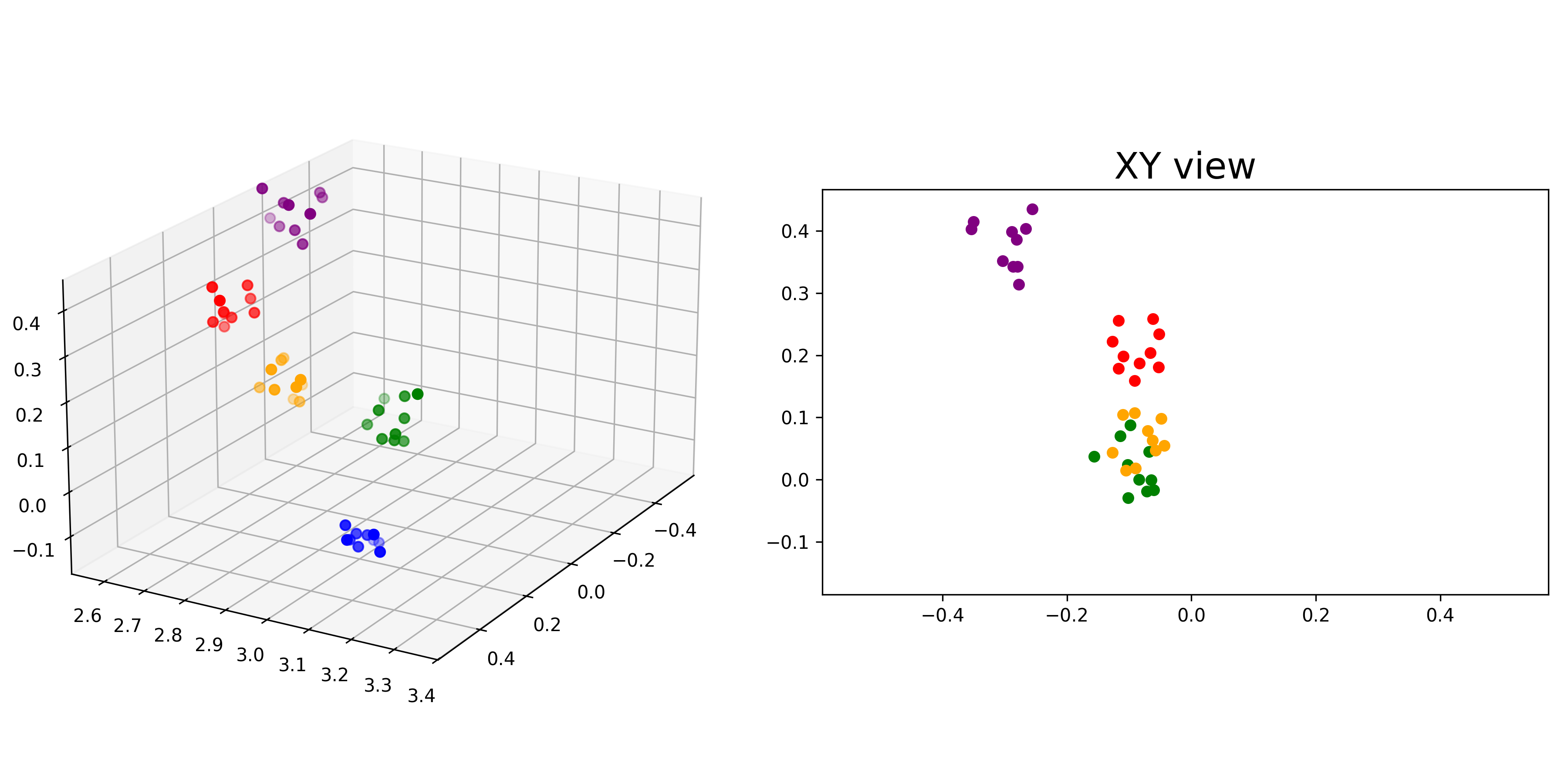}
        \caption{3D Point Cloud}
        \label{subfig:3d_spherical_sim}
    \end{subfigure}
    \begin{subfigure}[t]{0.4\linewidth}
    \includegraphics[width=\linewidth,trim={0 8 0 10},clip]{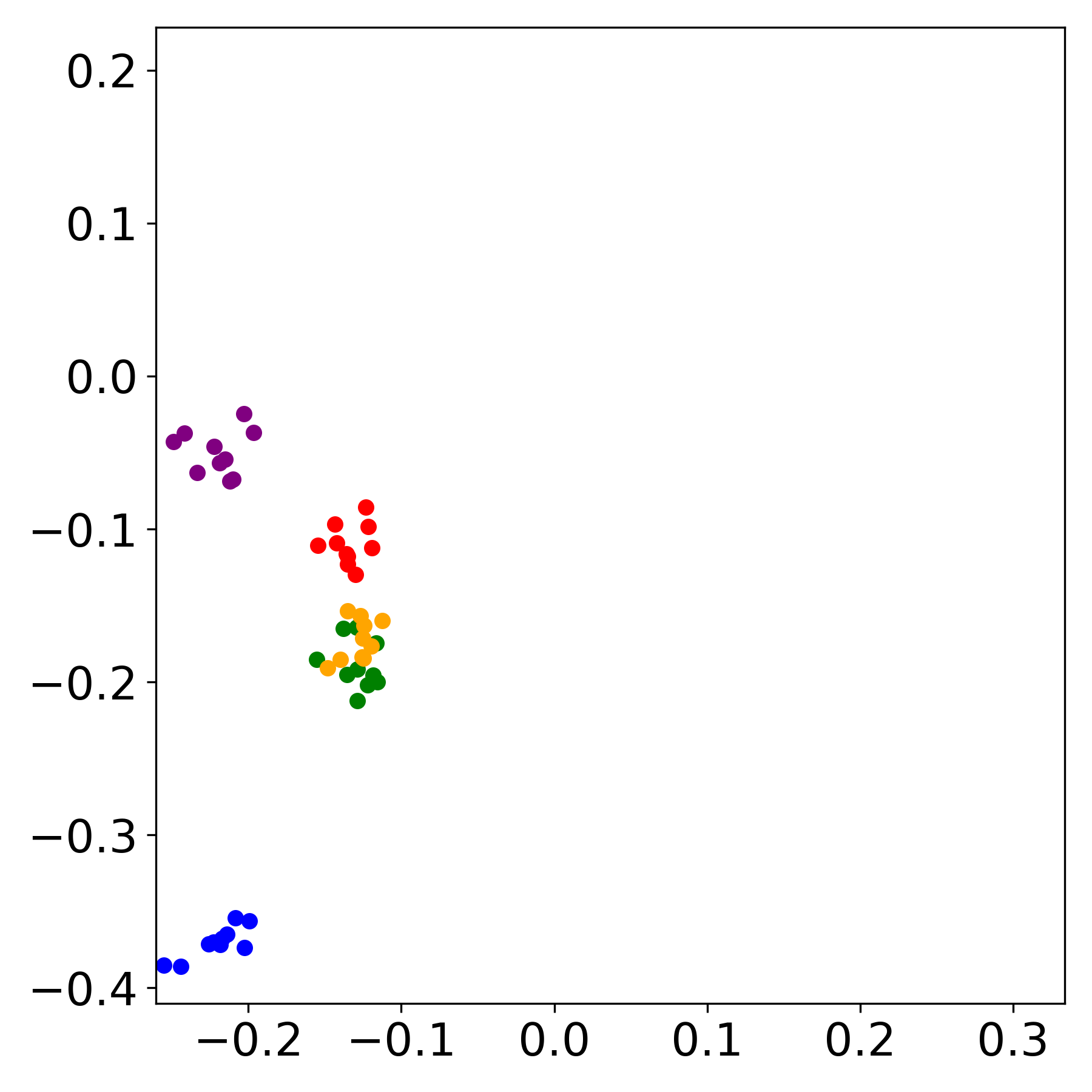}
        \caption{Left Camera}
            \label{subfig:left_spherical_sim}
    \end{subfigure}
    \quad
    \begin{subfigure}[t]{0.4\linewidth}
        \centering
    \includegraphics[width=\linewidth,trim={0 8 0 10},clip]{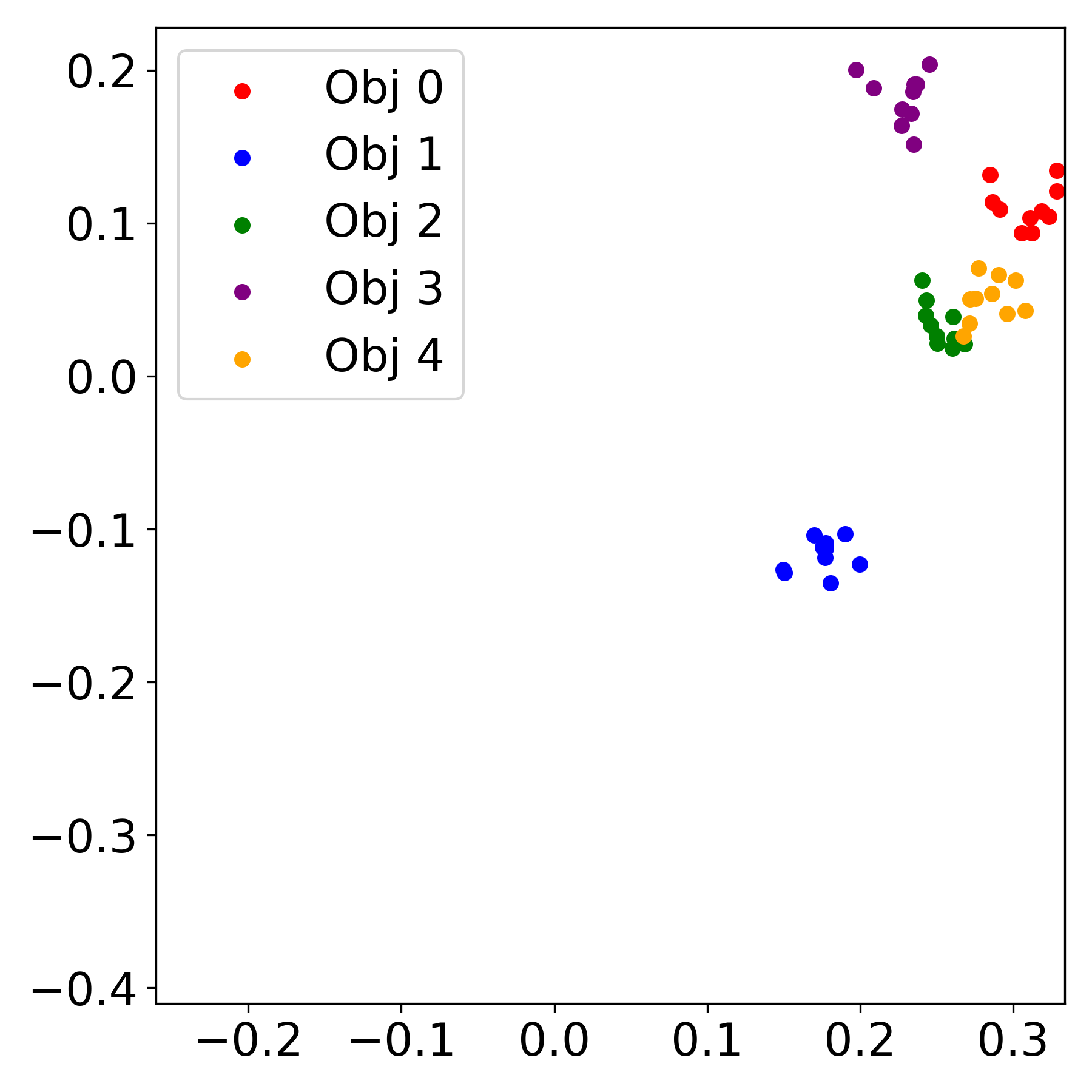}
        \caption{Right Camera}
            \label{subfig:right_spherical_sim}
    \end{subfigure}
    \caption{Spheres simulation with $\sigma=0.005$. (a): Original data visualized as a 3D scatter plot and as a 2D projection. 
    (b): Projection onto the left camera.
    (c): Projection onto the right camera. Shared axes of the camera planes in (b) and (c) highlight camera translation and rotation.}
    \label{Fig:spherical_simulation}
\end{figure}

\begin{table*}[t]
\scriptsize
\setlength{\tabcolsep}{2.0pt}
\centering
\begin{minipage}{0.48\linewidth}
\centering
\begin{tabular}{lccccc}
\toprule
Method & $\sigma=0.0$ & $0.001$ & $0.005$ & $0.01$ & $0.05$ \\
\midrule
$d^{\text{epi}}$ Naive & 0.0$\pm$0 & 44.1$\pm$11 & 80.6$\pm$8 & 88.5$\pm$6 & 95.9$\pm$3 \\
$d^{\text{epi}}$ OT & 0.0$\pm$0 & 35.9$\pm$10 & 78.4$\pm$7 & 87.8$\pm$5 & 95.7$\pm$3 \\
$d^{\text{epi}}$ HOT & 0.0$\pm$0 & 29.0$\pm$9 & 71.9$\pm$8 & 83.6$\pm$6 & 91.2$\pm$5 \\
\midrule
$d^{\text{ray}}$ Naive & 0.0$\pm$0 & 44.5$\pm$11 & 80.5$\pm$8 & 88.7$\pm$5 & 95.9$\pm$3 \\
$d^{\text{ray}}$ OT & 0.0$\pm$0 & 36.1$\pm$11 & 76.9$\pm$8 & 86.9$\pm$5 & 95.5$\pm$3 \\
$d^{\text{ray}}$ HOT & 0.0$\pm$0 & 29.5$\pm$9 & 71.4$\pm$8 & 82.6$\pm$6 & 91.3$\pm$5 \\
\midrule
$d^{\text{reg}}$ Naive & 0.3$\pm$1 & 37.5$\pm$11 & 75.7$\pm$8 & 85.2$\pm$6 & 94.3$\pm$4 \\
$d^{\text{reg}}$ OT & 0.3$\pm$1 & 27.3$\pm$11 & 68.8$\pm$10 & 82.0$\pm$6 & 93.4$\pm$4 \\
$d^{\text{reg}}$ HOT & 0.2$\pm$1 & 24.9$\pm$10 & 66.4$\pm$9 & 80.2$\pm$6 & 90.8$\pm$5 \\
\bottomrule
\end{tabular}
\caption{Pointwise mismatch ratio for simulated spheres (in \%, lower is better $\downarrow$).}
\label{tab:syn_point_mismatch}
\end{minipage}
\hfill
\begin{minipage}{0.48\linewidth}
\centering
\setlength{\tabcolsep}{2.0pt}
\begin{tabular}{lccccc}
\toprule
Method & $\sigma=0.0$ & $0.001$ & $0.005$ & $0.01$ & $0.05$ \\
\midrule
$d^{\text{epi}}$ Naive & 0.0$\pm$0 & 1.2$\pm$5 & 2.0$\pm$8 & 2.3$\pm$6 & 12.3$\pm$29 \\
$d^{\text{epi}}$ OT & 0.0$\pm$0 & 1.1$\pm$5 & 1.6$\pm$4 & 2.1$\pm$5 & 8.5$\pm$15 \\
$d^{\text{epi}}$ HOT & 0.0$\pm$0 & 0.1$\pm$0 & 0.1$\pm$0 & 0.4$\pm$2 & 5.5$\pm$18 \\
\midrule
$d^{\text{ray}}$ Naive & 0.0$\pm$0 & 1.2$\pm$5 & 1.4$\pm$3 & 2.1$\pm$5 & 7.8$\pm$16 \\
$d^{\text{ray}}$ OT & 0.0$\pm$0 & 0.9$\pm$4 & 1.1$\pm$3 & 1.1$\pm$2 & 4.0$\pm$12 \\
$d^{\text{ray}}$ HOT & 0.0$\pm$0 & 0.1$\pm$0 & 0.1$\pm$0 & 0.2$\pm$0 & 0.8$\pm$0 \\
\midrule
$d^{\text{reg}}$ Naive & 0.0$\pm$0 & 0.1$\pm$0 & 0.2$\pm$0 & 0.3$\pm$0 & 5.8$\pm$15 \\
$d^{\text{reg}}$ OT & 0.0$\pm$0 & 0.1$\pm$0 & 0.1$\pm$0 & 0.1$\pm$0 & 0.6$\pm$2 \\
$d^{\text{reg}}$ HOT & 0.0$\pm$0 & 0.1$\pm$0 & 0.1$\pm$0 & 0.1$\pm$0 & 0.6$\pm$1 \\
\bottomrule
\end{tabular}
\caption{Squared Wasserstein-2 for 3D reconstruction of simulated spheres, where lower is better ($\downarrow$).}
\label{tab:syn_point_wasserstein}
\end{minipage}
\end{table*}

\begin{table}[t]
\scriptsize
\centering
\setlength{\tabcolsep}{5pt}
\begin{tabular}{lccccc}
\toprule
$c$ & $\sigma=0.0$ & $0.001$ & $0.005$ & $0.01$ & $0.05$ \\
\midrule
$d^{\text{epi}}$ & 0.0$\pm$0 & 0.0$\pm$0 & 0.8$\pm$6 & 3.2$\pm$11 & 19.2$\pm$24 \\
\midrule
$d^{\text{ray}}$ & 0.0$\pm$0 & 0.0$\pm$0 & 0.4$\pm$4 & 2.0$\pm$9 & 10.8$\pm$19 \\
\midrule
$d^{\text{reg}}$ & 0.0$\pm$0 & 0.0$\pm$0 & 0.8$\pm$6 & 1.2$\pm$7 & 15.2$\pm$22 \\
\bottomrule
\end{tabular}
\caption{Object mismatch using HOT for simulated spheres (in \%, $\downarrow$).}
\label{tab:syn_object_mismatch}
\end{table}

\paragraph*{Simulation}
We generate 100 synthetic 3D scenes composed of $N = 5$ disjoint spherical objects. 
For each object $k \in [5]$ we sample a center 
$c_k~\in~[-0.5,0.5]^2~\times~[2.5, 3.5]$ uniformly at random. The radius $r_k$ of each sphere is drawn from a uniform distribution over $[0.05,0.1]$. We reject and resample any proposed center whose sphere would overlap 
with an already placed sphere. 
On each sphere, we uniformly sample 10 points. 
This yields a point cloud of $50$ points grouped into $5$ spatially separated objects for each synthetic scene.
We project all point clouds onto two cameras located at $(0, 0, 0)$ and $(1, 0, 0)$, set $K$ as the identity matrix, and employ random camera rotations up to $\pm 15$ degrees to both. By Remark \ref{rem:rot}, camera rotations do not impact the ray distance.
We repeat all experiments with varying levels of independent Gaussian noise $\mathcal{N}(0, \sigma^2)$ added to each point in both projections. An example is visualized in Figure~\ref{Fig:spherical_simulation}.

\paragraph*{Point Matching via OT and HOT}
We investigate the point-to-point matching quality for 
all combinations of distances and pointwise matching, i.e.,
for combinations of \eqref{eq:line_corrected_distance}, the depth-regularized ray distance \eqref{eq:corrected_distance_line_z} ($\gamma_1=2.5$, $\gamma_2=3.5$, $\beta=10$), and the epipolar distance~\eqref{eq:epipolar_distance}
with the naive, OT~\eqref{eq:OT_one_to_one}, and HOT matching~\eqref{eq:global_plan}. As described in Section~\ref{sec:matching_rate}, we evaluate our matching for different noise levels 
based on the mismatch ratio in Table~\ref{tab:syn_point_mismatch} and  the Wasserstein distance between the ground truth 3D point cloud and the reconstruction
in Table~\ref{tab:syn_point_wasserstein}.

{\color{black}
Although pointwise matching degrades quickly under noise, OT plans still mostly align points within the correct object, leading to low object-level mismatch rates for mild to moderate noise, see Table~\ref{tab:syn_object_mismatch}.
This robustness stems from the scene’s structure, and points are tightly clustered per object, so perturbations mainly affect intra-object rather than inter-object matches. Among the evaluated distances, ray and epipolar formulations perform similarly in mismatch rates, while ray distance yields better Wasserstein scores. Adding depth regularization further improves reconstruction quality under noise.
Finally, HOT strengthens this behavior by enforcing group consistency and preventing mass splitting across objects, which particularly enhances the results for $d^\text{ray}$ and $d^\text{epi}$.
}

\paragraph*{Object Matching via HOT}
Using the HOT approach, we further investigate the resulting sphere-to-sphere matching in Table~\ref{tab:syn_object_mismatch}. {Even when the pointwise matching fails due to noise}, we obtain stable matching.
{\color{black}
The relatively large standard deviations are mainly due to the small number of objects (5 per scene), which quantizes the accuracy in steps of 20\%. As a result, a single mismatch induces a 20\% change, leading to fluctuations such as $2\% \pm 9\%$ corresponding to 0–1 mismatched objects across runs. This is a discretization effect rather than a sign of instability in the method.
}
Overall, we get the best results with the ray distance.

\subsection{Matching RGB and Thermal Landmarks} \label{sec:realworld}

\paragraph*{Dataset}
Our setup consists of two calibrated cameras with known intrinsic and extrinsic parameters,
capturing frontal views of a human subject, see Figure~\ref{fig:thermal}. The calibrated images were obtained during a study at Saarland University, see~\cite{flotho2021multimodal,flotho2023lagrangian}. We consider the 468 Mediapipe landmarks \cite{lugaresi2019mediapipe} based on the first RGB camera as our left point cloud.
For the right point cloud, we consider either the Mediapipe landmarks on the second RGB camera (\enquote{RGB-RGB}) or 5/70/478 landmarks from the thermal camera (\enquote{RGB-Thermal}), obtained via the landmarkers from \cite{flotho2025t} and \cite{TFW}.

\paragraph*{RGB-RGB Point Matching via OT}
In the RGB-RGB setup, we aim to match 468 Mediapipe landmarks with known ground-truth correspondence. Unlike the synthetic faces from Section~\ref{sec:Synth_Faces}, our calibration parameters and our landmark projections are subject to real-world noise. As a result, our OT matching 
based on the ray distance~\eqref{eq:line_corrected_distance} leads to points being matched at practically infinite distance from the cameras. This highlights the advantage of the regularized ray distance~\eqref{eq:corrected_distance_line_z}, where we employ the parameters $\gamma_1 = 1550$, $\gamma_2 = 1750$ penalizing the depth of the scene, and the regularization strength parameter $\beta = 100$. 
The 3D reconstructions with and without regularization are shown in Figure~\ref{fig:3d_recon_real_rgb}. While the unregularized ray distance results in a poor reconstruction, the regularized distance {reconstructs the shape of the face}.

\begin{figure}[!ht]
\centering
\begin{subfigure}[t]{0.32\linewidth}
    \centering
    \includegraphics[width=\linewidth, trim={0 0 450 0}, clip]{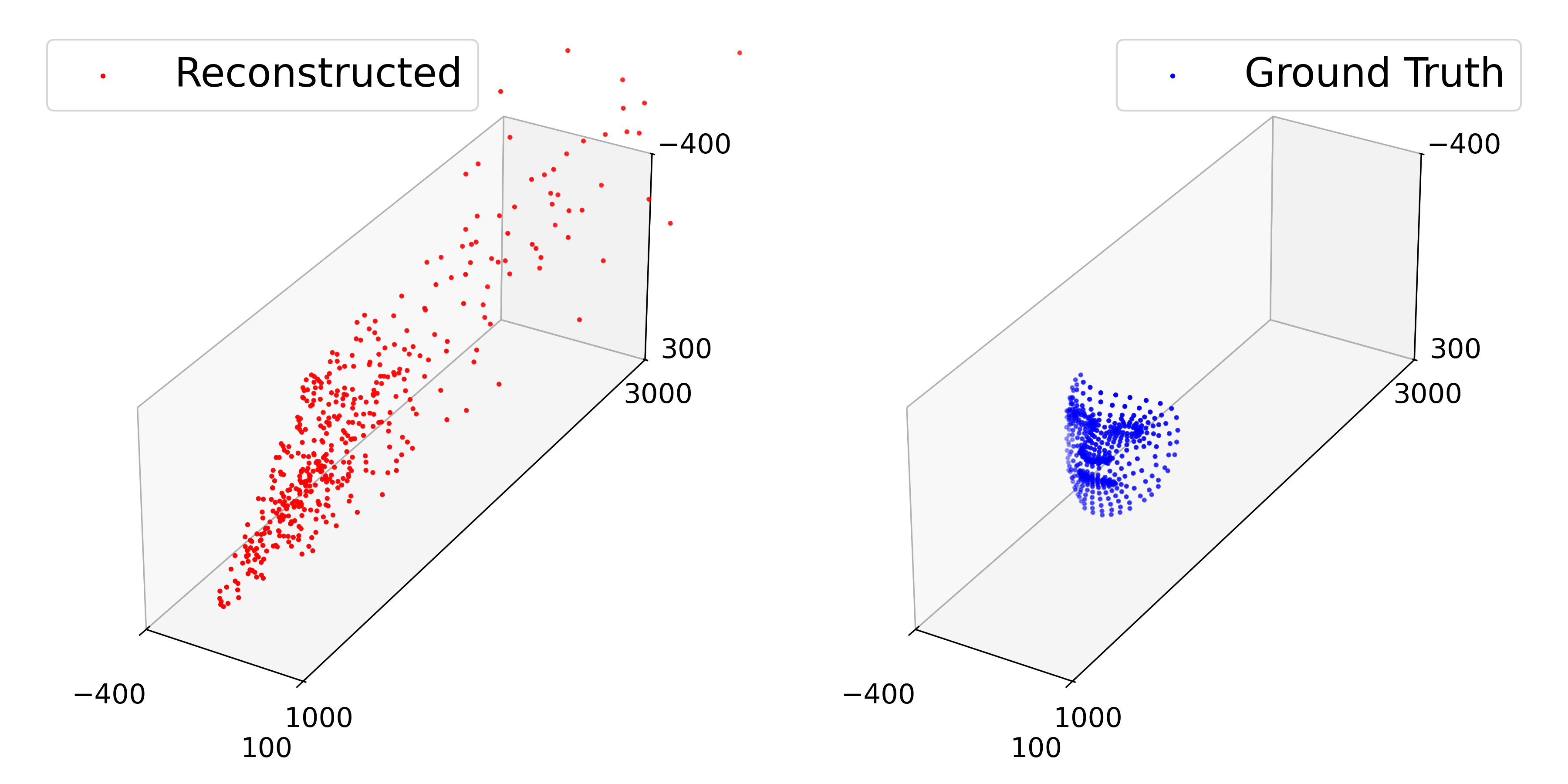}
    \caption{$d^\mathrm{ray}$}
    \label{subfig:unreg_3d}
\end{subfigure}\hfill
\begin{subfigure}[t]{0.32\linewidth}
    \centering
    \includegraphics[width=\linewidth, trim={0 0 450 0}, clip]{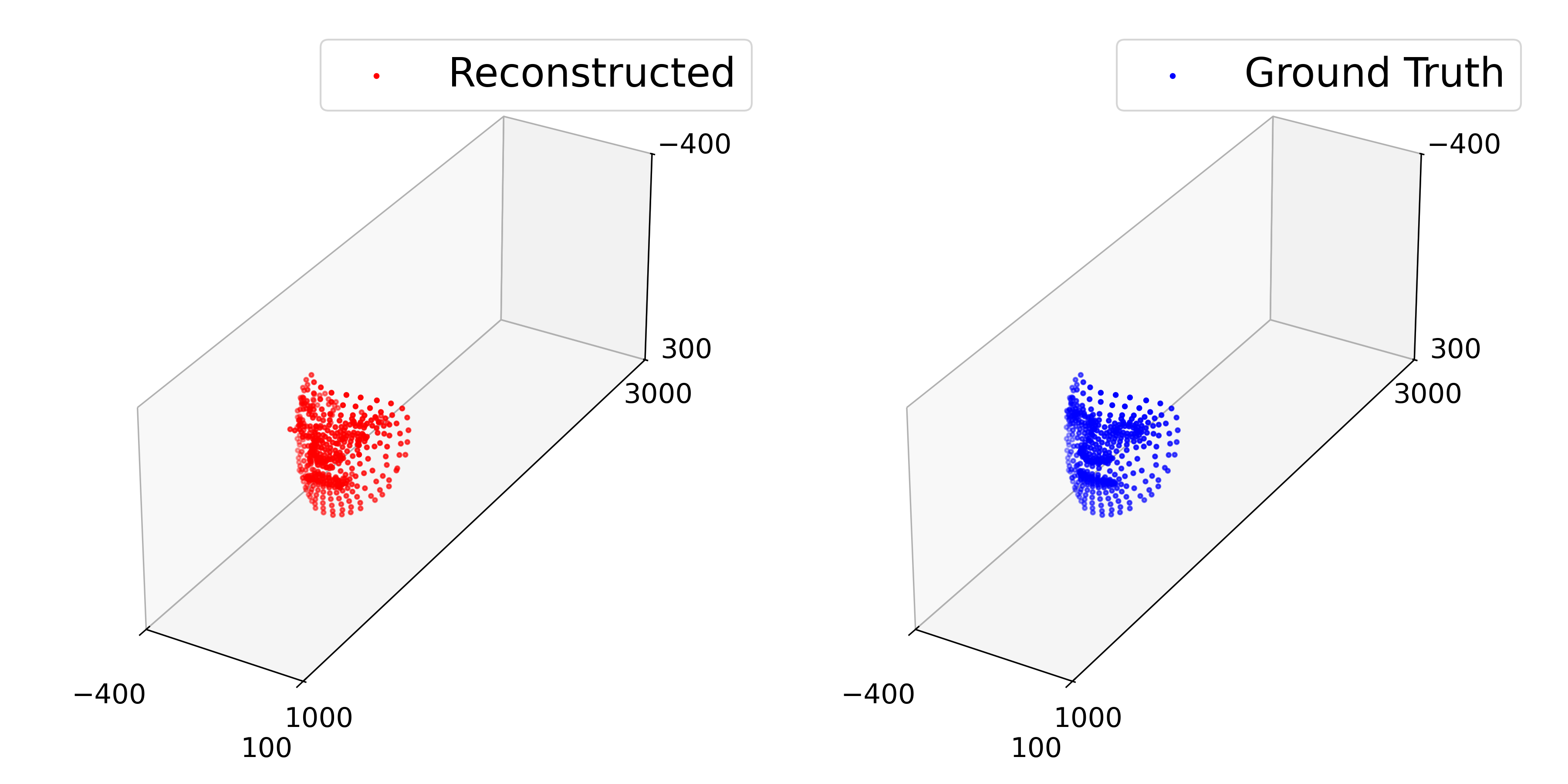}
    \caption{$d^\mathrm{reg}$}
    \label{subfig:reg_3d}
\end{subfigure}\hfill
\begin{subfigure}[t]{0.32\linewidth}
    \centering
    \includegraphics[width=\linewidth, trim={450 0 0 0}, clip]{Images_new/real_rgb_rgb_matching/reconstruction_rgb_regularized.png}
    \caption{GT}
    \label{subfig:gt_3d}
\end{subfigure}
\caption{3D reconstruction of a face using matched Mediapipe landmarks (a) with the ray distance, (b) with the depth-regularized ray distance and (c) the ground truth (GT). While the first one gives poor results, regularization prevents matching with unreasonable depth.}
\label{fig:3d_recon_real_rgb}
\end{figure}

\paragraph*{RGB-Thermal Point Matching via POT}
For the matching of RGB and thermal landmarks, there are no true correspondences, and a one-to-one matching is not possible for the 478-point (Fig.~\ref{fig:ther_full}), the 70-point (Fig.~\ref{fig:ther_sub}), or the 5-point convention (Fig.~\ref{fig:ther_5}). {\color{black} We perform a partial matching between the different landmark conventions using POT \eqref{eq:POT_matching}, with $m=\frac{\min\{N, M\}}{\max\{N, M\}}$, and a regularized ray distance \eqref{eq:corrected_distance_line_z} with $\gamma_1=1550$, $\gamma_2=1750$, and $\beta=100$. The regularization parameters $\gamma_{1/2}$ were chosen based on prior knowledge of the expected landmark positions, and $\beta$ was chosen sufficiently large to ensure that incorrect matches are quickly discarded}. We visualize the results in Figure~\ref{fig:thermal_landmarks_projection} by projecting the thermal landmarks onto the RGB camera plane using the known calibration parameters and connecting matched points. Qualitatively, we see good matching with corresponding facial keypoints paired correctly across modalities. 

{\color{black}
As shown in Table \ref{tab:RGB_dense}, using non-optimal regularization parameters $\gamma$ degrades matching performance, with reliable results only in a small range. The method is more sensitive to $\gamma_1$ than to $\gamma_2$, with stable behavior as long as $\gamma_1$ remains near its optimal value.
The matching of faces rather than points is much more stable,
as we will see in Section~\ref{sec:RGB_thermal}.
}

\begin{table}[h]
\centering
{\color{black}
\begin{tabular}{c|ccccc}
\hline
$\gamma_1 \backslash \gamma_2$ & 1750 & 1800 & 2000 & 2500 & 5000 \\
\hline
0.0    & 24.264 & 28.577 & 39.624 & 49.515 & 52.514 \\
500.0  & 24.264 & 28.577 & 39.624 & 49.515 & 52.514 \\
1000.0 & 24.264 & 28.577 & 39.624 & 49.515 & 52.514 \\
1250.0 & 24.313 & 28.206 & 37.885 & 47.634 & 49.685 \\
1500.0 & 15.479 & 16.657 & 18.876 & 19.414 & 19.414 \\
1550.0 & \textbf{10.893} & \underline{11.515} & 12.238 & 12.389 & 12.389 \\
\hline
\end{tabular}
}
\caption{Mean distance between matched points in the dense thermal--RGB configuration and the optimal RGB--RGB configuration, where lower is better ($\downarrow$), for different values of $\gamma_1$ and $\gamma_2$. Best results and second best are respectively highlighted in bold and underlined}
\label{tab:RGB_dense}
\end{table}

\begin{figure}[!ht]
    \centering
    \begin{subfigure}[t]{0.32\linewidth}
        \centering
        \includegraphics[width=\linewidth]{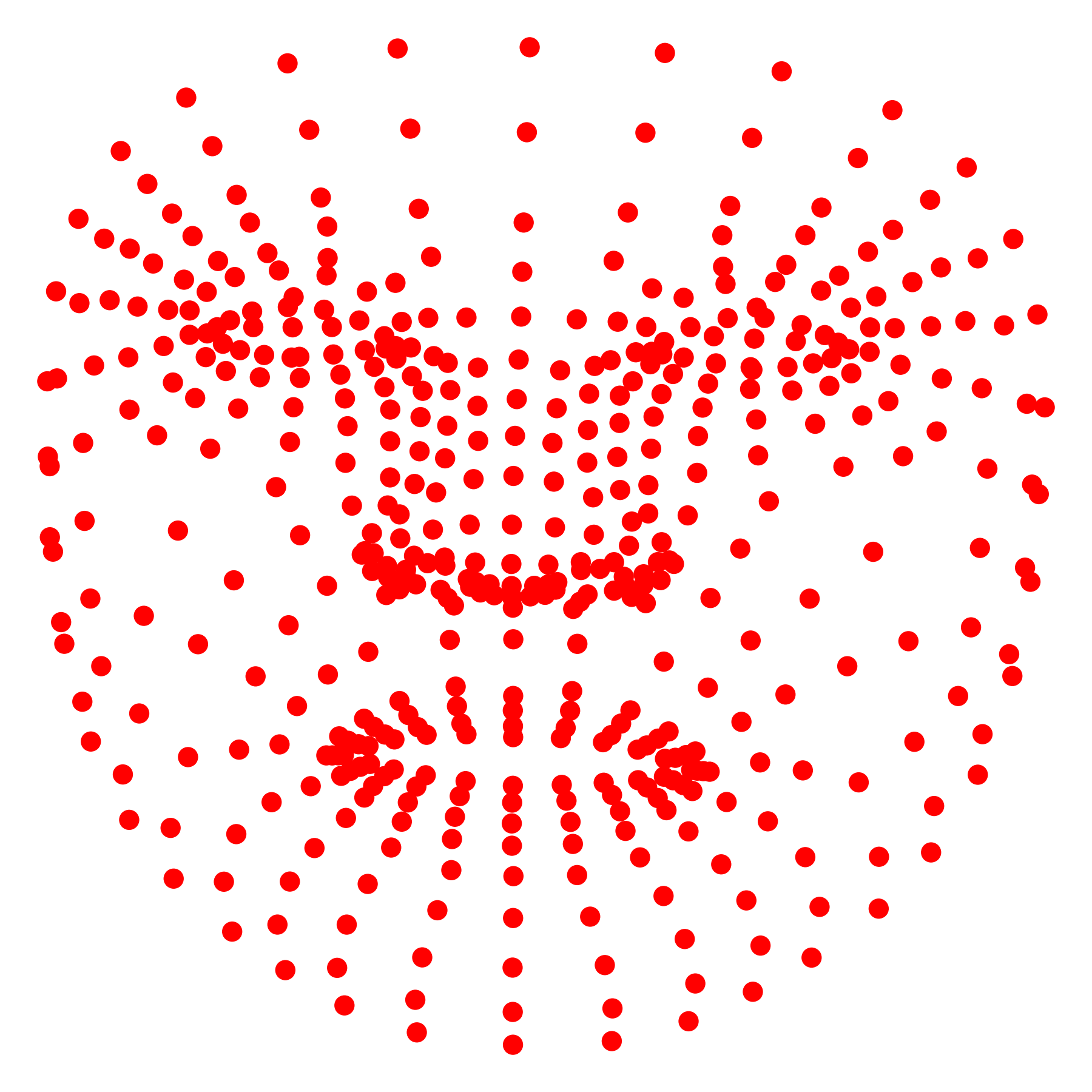}
        \caption{{478} landmarks}
        \label{fig:ther_full}
    \end{subfigure}
    \hfill
    \begin{subfigure}[t]{0.32\linewidth}
        \centering
        \includegraphics[width=\linewidth]{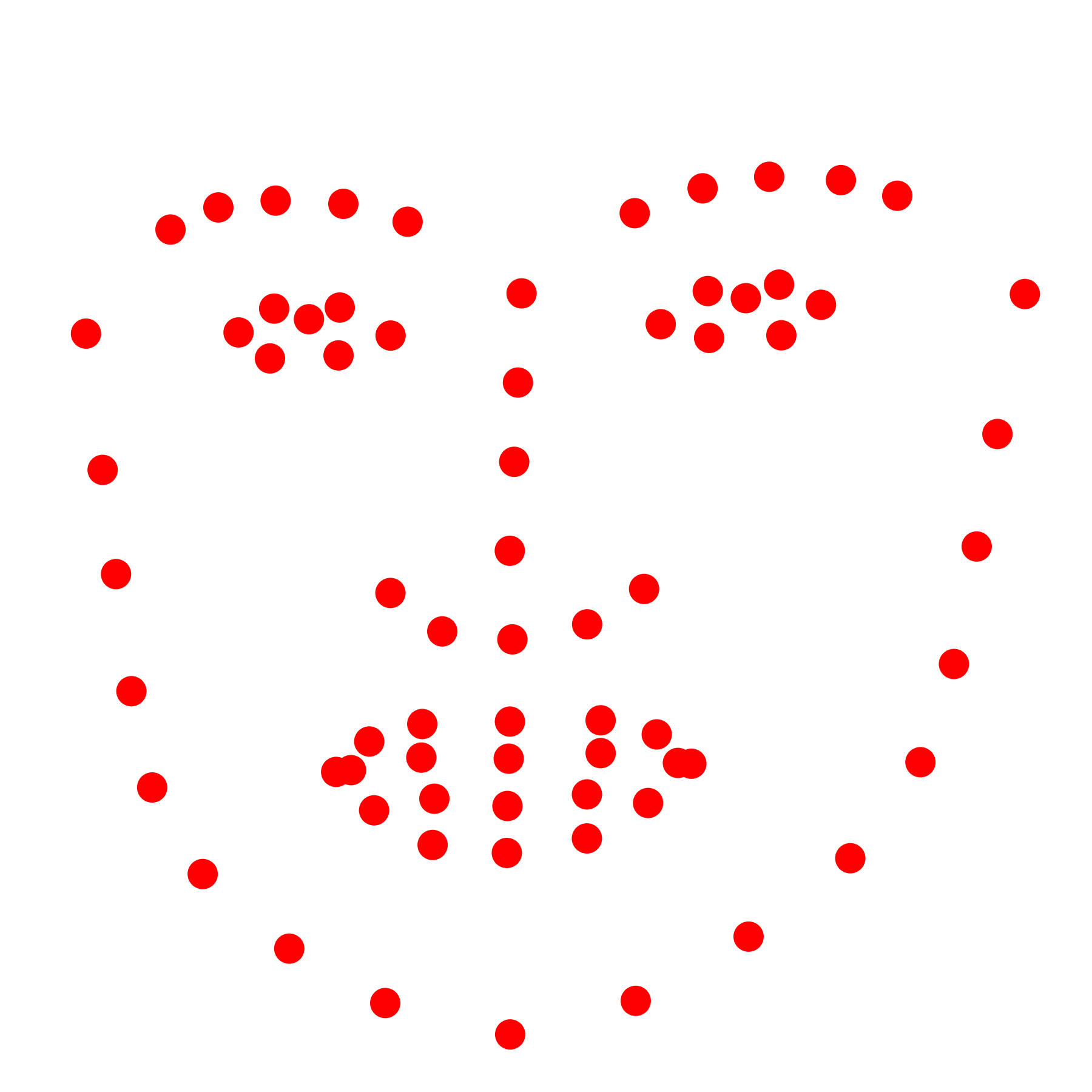}
        \caption{70 landmarks}
        \label{fig:ther_sub}
    \end{subfigure}
    \hfill
    \begin{subfigure}[t]{0.32\linewidth}
        \centering
        \includegraphics[width=\linewidth]{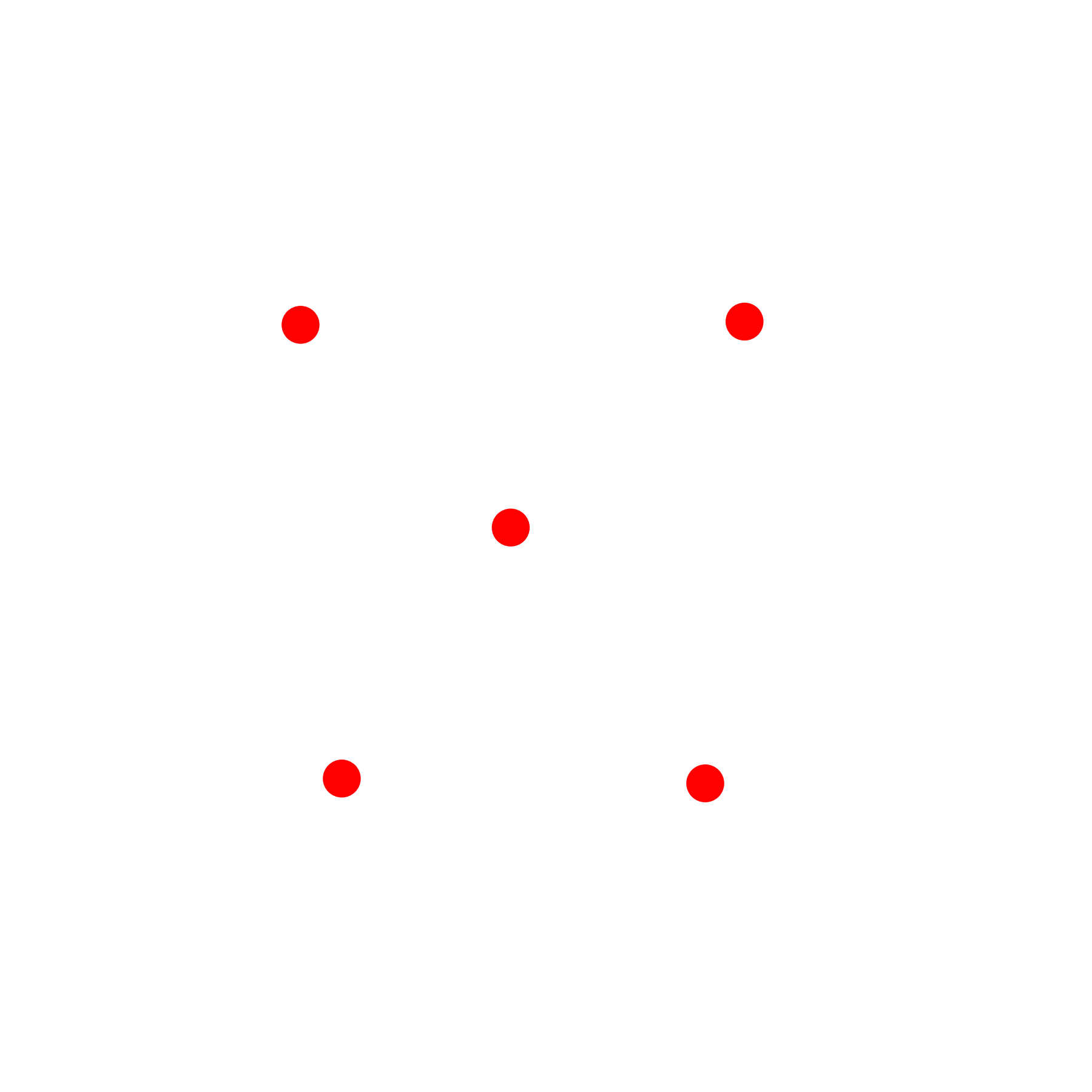}
        \caption{5 landmarks}
        \label{fig:ther_5}
    \end{subfigure}

    \vspace{0.4cm}

    \begin{subfigure}[t]{0.32\linewidth}
        \centering
        \includegraphics[width=\linewidth]{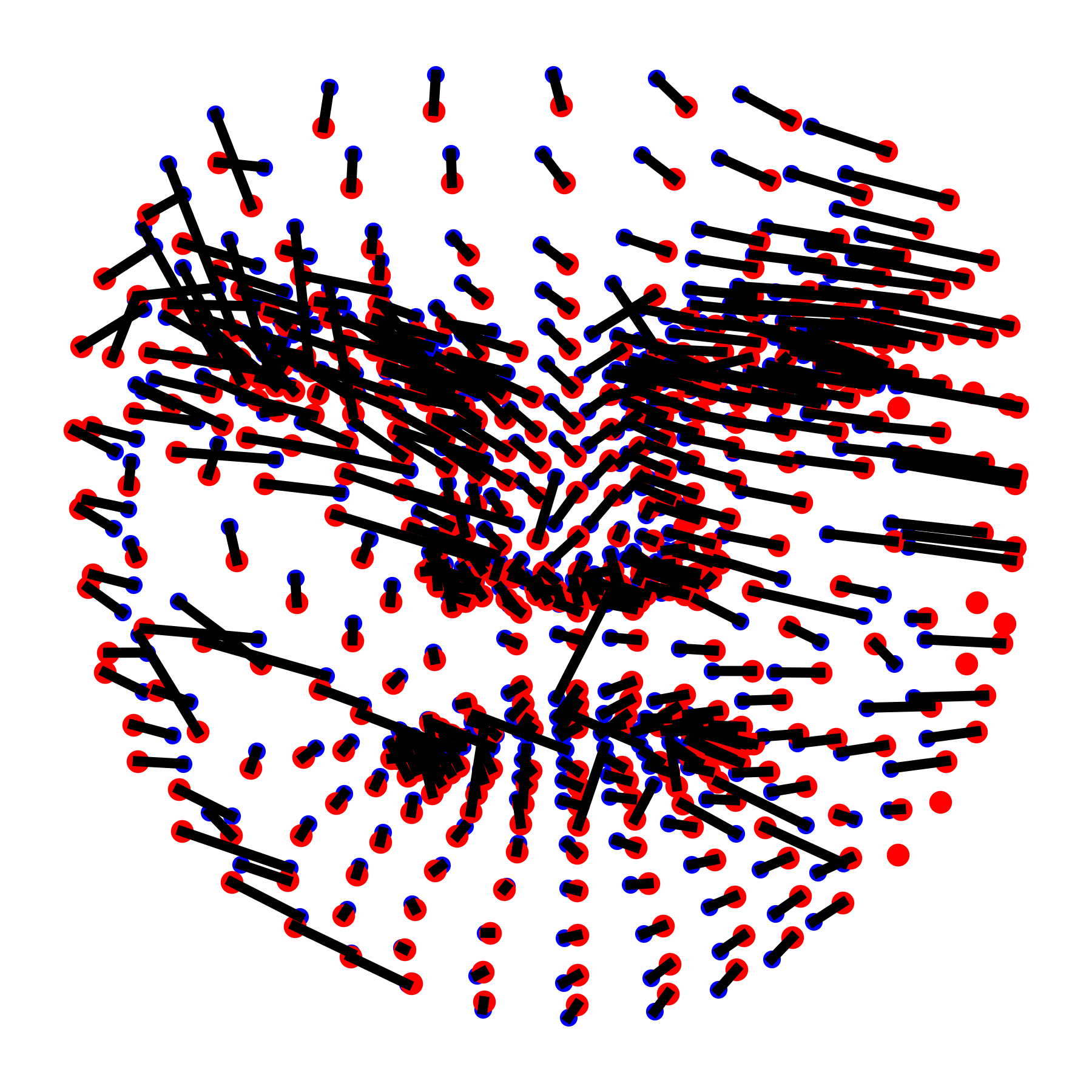}
        \caption{478 projected landmarks}
        \label{fig:matching_dense_proj}
    \end{subfigure}
    \hfill
    \begin{subfigure}[t]{0.32\linewidth}
        \centering
        \includegraphics[width=\linewidth]{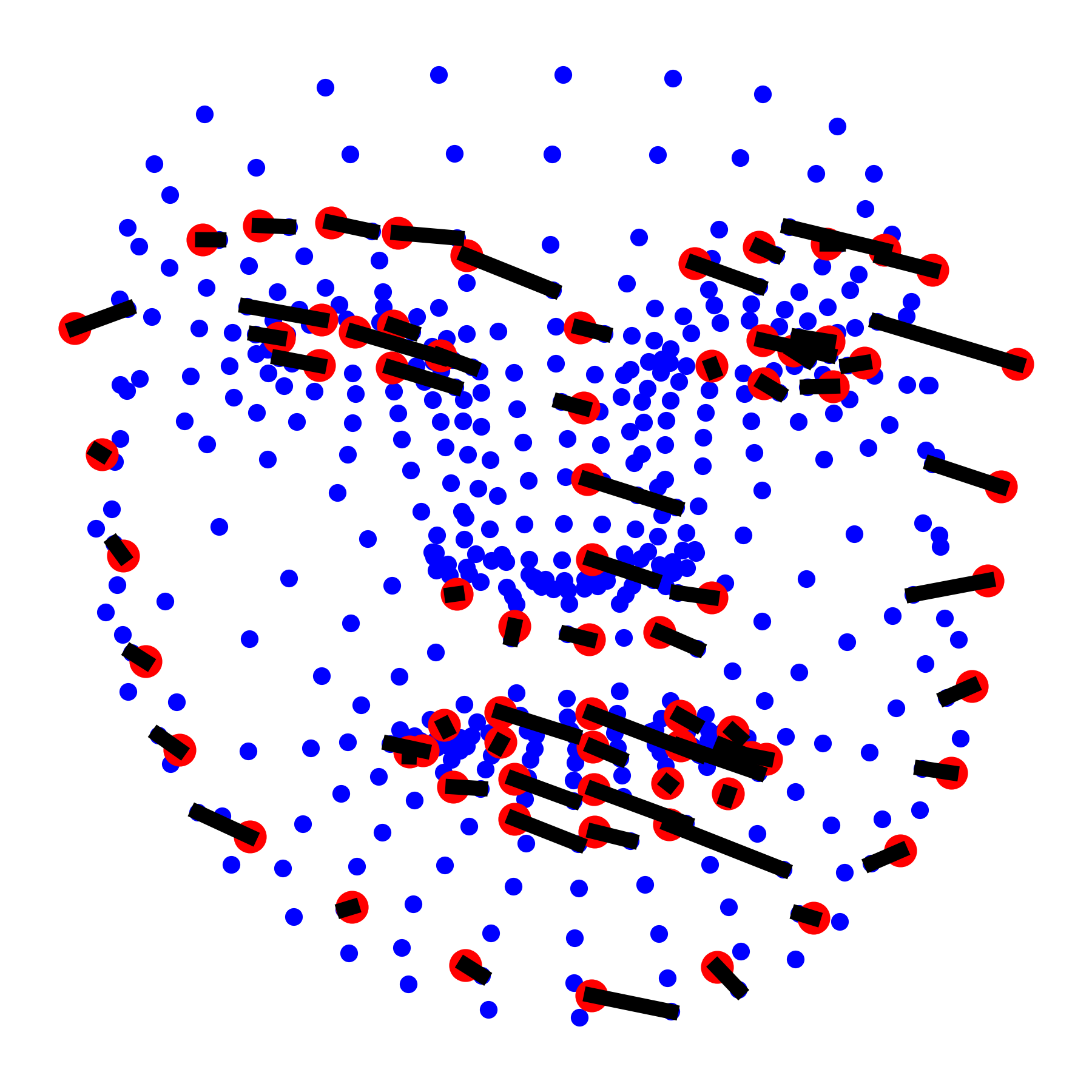}
        \caption{70 projected landmarks}
        \label{fig:matching_sparse_proj}
    \end{subfigure}
    \hfill
    \begin{subfigure}[t]{0.32\linewidth}
        \centering
        \includegraphics[width=\linewidth]{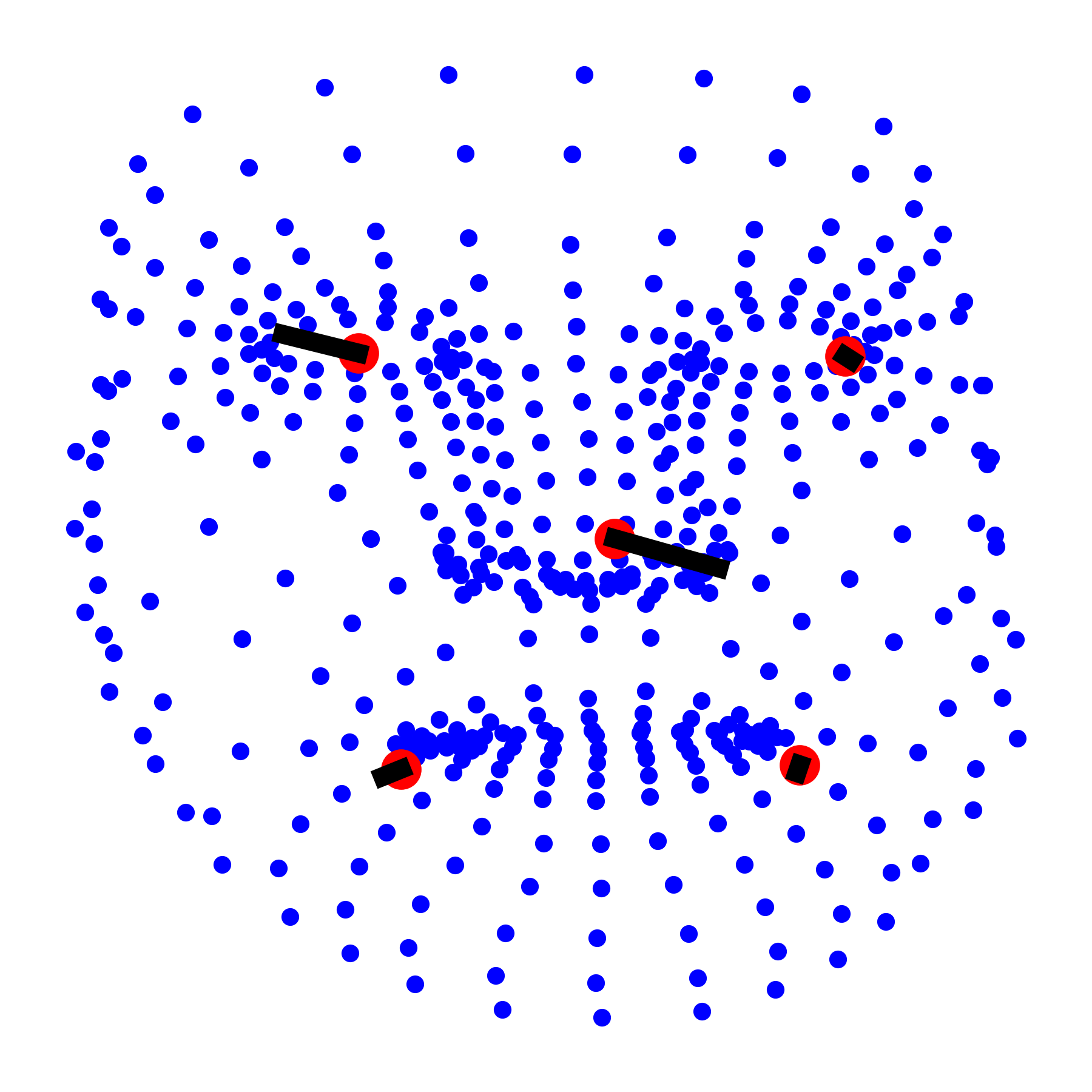}
        \caption{5 projected landmarks}
        \label{fig:matching_5pts_proj}
    \end{subfigure}

    \caption{Top row: 2D right thermal camera images for different thermal landmark conventions. Bottom row: projection of matched thermal landmarks (red) onto the left RGB camera image. Light blue lines link thermal landmarks to the matched RGB landmarks (blue).}
    \label{fig:thermal_landmarks_projection}
\end{figure}

\subsection{Matching RGB and Thermal Faces} \label{sec:RGB_thermal}

\paragraph*{Dataset}
Finally, we evaluate our HOT-POT approach for cross-modal face matching with various real-world measurements.
We use an RGB–thermal video recorded by the Systems Neuroscience and Neurotechnology Unit (SNNU) at Saarland University, showing three persons moving around a room. 
We extract 20 frame pairs from the videos and detect 468 RGB landmarks per face using Mediapipe \cite{lugaresi2019mediapipe} and 70 thermal landmarks using the T-FAKE landmaker \cite{flotho2025t} in combination with the TFW face tracker \cite{TFW}.
Notably, some faces are occluded in some frames, and the Mediapipe landmarker does not always detect every face, leading to a varying number of faces per frame and per camera.
Here, the frames were chosen such that the time difference is small (both cameras have a different frame rate), and to ensure that each camera detects at least one face and at least one camera detects more than one face. 

The camera calibration parameters are estimated from 59 calibration frames provided by SNNU, showing an asymmetric circle-grid target observed at different positions and orientations. Calibration is performed using standard routines from OpenCV~\cite{opencv_library}.

Examples are visualized in Figure~\ref{fig:real_examples_rgb_thermal}.
There are various sources of errors resulting from
i) the estimated camera calibration, 
ii) the inconsistent number of landmarks and faces between both cameras,
iii) temporal delays between the two camera frames,
and
iv) landmark detection errors.

\paragraph*{Face Matching via HOT-POT}
Based on the estimated camera calibration parameters, we run our HOT-POT algorithm for face matching using the ray distance~\eqref{eq:line_corrected_distance}, the depth-regularized ray distance~\eqref{eq:corrected_distance_line_z} with $\gamma_1 = 500$, $\gamma_2 = 5000$, and $\beta = 1$, as well as the epipolar distance~\eqref{eq:epipolar_distance}. {\color{black} In this real-world setting, prior information is inherently less precise due to noise and calibration uncertainty, which motivates a less strict choice of hyperparameters, including a smaller $\beta$ to reduce the impact of unreliable matches.}

\begin{figure}[ht]
    \centering

    \begin{subfigure}{0.48\linewidth}
        \centering
        \includegraphics[width=\linewidth]{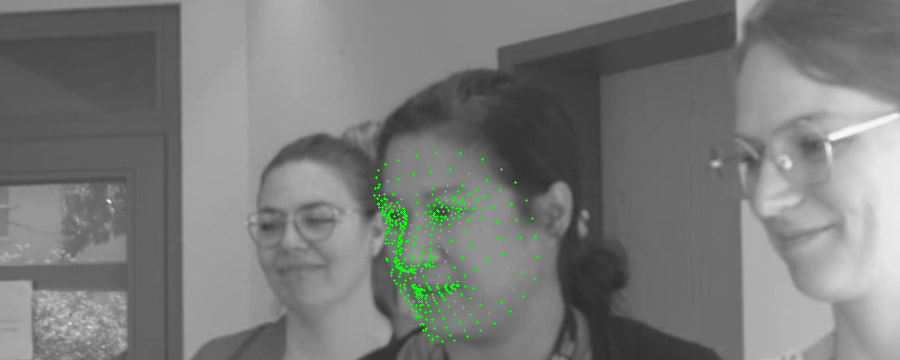}
        \caption{RGB 1}
    \end{subfigure}
    \hfill
    \begin{subfigure}{0.48\linewidth}
        \centering
        \includegraphics[width=\linewidth]{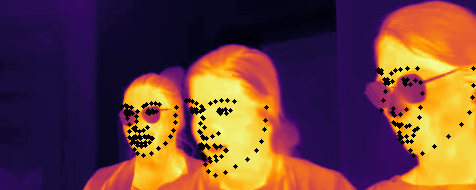}
        \caption{Thermal 1}
    \end{subfigure}
    \vspace{0.5em}
    \begin{subfigure}{0.48\linewidth}
        \centering
        \includegraphics[width=\linewidth]{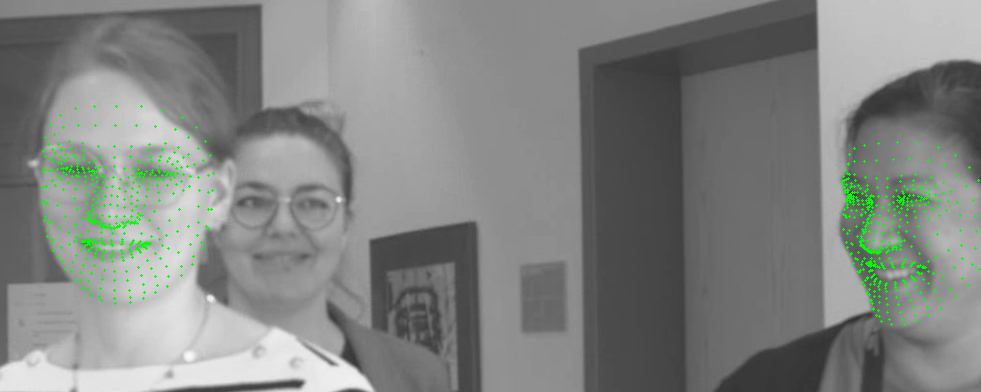}
        \caption{RGB 2}
    \end{subfigure}
    \hfill
    \begin{subfigure}{0.48\linewidth}
        \centering
        \includegraphics[width=\linewidth]{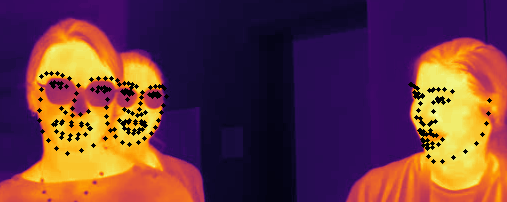}
        \caption{Thermal 2}
    \end{subfigure}
    
    \caption{Two cropped RGB and thermal landmark pairs extracted from our video provided by SNNU. For visualization, RGB images are shown in grayscale.
    Between RGB and thermal camera, there are small time lags and the number of detected faces may differ.}
    \label{fig:real_examples_rgb_thermal}
\end{figure}

{Frames are marked as correct if all landmarked faces are correctly matched and the mismatch rate is averaged over 20 frame pairs.}
As shown in Table~\ref{tab:distance_comparison}, the epipolar distance $d_{\text{epi}}$ incorrectly matches 24\% of all available face pairs, resulting in errors in 5 out of 20 frames. In contrast, the ray distance $d_{\text{ray}}$ achieves a mismatch rate of 5\%, with only a single erroneous frame. The depth-regularized ray distance $d_{\text{reg}}$ performs best, correctly matching all pairs of faces in all 20 frames.

\begin{table}[ht]
\centering
\begin{tabular}{lccc}
\hline
Method & Correct Frames & Mismatch rate (\%) \\
\hline
$d_{\text{ray}}$ & 19/20 & 5\% \\
$d_{\text{reg}}$ & 20/20 & 0\% \\
$d_{\text{epi}}$ & 15/20 & 24\% \\
\hline
\end{tabular}
\caption{Comparison of distance metrics for the real RGB-thermal data.}
\label{tab:distance_comparison}
\end{table}

\section{Conclusions} \label{sec:conclusions}

We proposed a new approach for 3D stereo matching of sparse point clouds using a partial OT framework, 
where the matching costs are derived from epipolar geometry. 
While our first cost is based on the 3D distance between the rays through the camera plane and the focal point, 
our second cost relies on enforcing the epipolar constraints.
The ray-based cost, combined with a regularization term, provides more robust performance than the commonly used epipolar constraint-based cost, especially in noisy settings.
For matching objects rather than single points, we developed a hierarchical matching framework, which first solves the POT between all possible object pairs and then calculates the matching among the objects. 
While we found a large sensitivity to measurement noise for the pointwise approach, the HOT matches the objects correctly in the case of large deviations and real-world measurements.
   
In the future, we want to extend our method to perform a three-way matching via multimarginal OT \cite{Ba2022,Lin2020,Pass2014},
integrate keypoint features such as color, and
incorporate our methods into dense stereo matching algorithms such as H-Net \cite{huang2022h}. 
Our application may become useful for public health screening, where one is interested in identifying persons with elevated temperature to prevent the spreading of infectious diseases.

\section*{Acknowledgements}
{AC gratefully acknowledges funding from the Berlin Mathematical School and support from the French Research Agency through the Holibrain project (ANR-23-CE45-0020).}
MQ gratefully acknowledges funding by the German Research Foundation (DFG): STE 571/19-1, project number 495365311, within the Austrian Science Fund (FWF) SFB 10.55776/F68 “Tomography Across the Scales”. 
MP gratefully acknowledges funding from the German Research Foundation (DFG) within the project BIOQIC
(GRK2260/289347353).
{PF acknowledges support from the Japan Society for the Promotion of Science (JSPS) under Grant No. PE25715.}
GK acknowledges funding by the BMBF VI\-ScreenPRO (ID: 100715327).
For open access purposes, the author has applied a CC BY public copyright license to any author-accepted manuscript version arising from this submission.

We are especially grateful to Mayur Bhamborae of the Systems Neuroscience and Neurotechnology Unit (SNNU) at Saarland University for providing the calibration data and frames, as well as the real RGB–thermal data used in our numerical experiments {and to Daniel J. Strauss for enabling the scientific exchange between Saarland University and TU Berlin.}

\bibliography{bib}
\newpage

\appendix
\end{document}